\pdfoutput=1
\documentclass[11pt]{arvix}

\usepackage[numbers]{natbib}
\usepackage{bm}
\usepackage{nicefrac}
\usepackage{multirow}
\usepackage{adjustbox}
\usepackage{array}
\usepackage{tablefootnote}
\usepackage{algorithm}
\usepackage{algorithmicx}
\usepackage[noend]{algpseudocode}
\usepackage{listings}
\usepackage{wrapfig}

\newtheorem{theorem}{Theorem}
\newtheorem{lemma}{Lemma}
\newtheorem{corollary}{Corollary}
\newtheorem{remark}{Remark}

\renewcommand{\eqref}[1]{Eq.~(\ref{#1})}

\providecommand{\AutoGEO}{\textsc{AutoGEO}}

\providecommand{\GEO}{\textsc{GEO}}

\providecommand{\AutoGEO}{\textsc{AutoGEO}}

\providecommand{\GEO}{\textsc{GEO}}
\providecommand{\VCR}{\textsc{VCR}}

\providecommand{\AttackerRule}[1]{\textcolor[HTML]{196f3d}{\textbf{#1}}}
\providecommand{\GenericRule}[1]{\textcolor[HTML]{777777}{\textit{#1}}}
\providecommand{\VCRtag}[1]{\textcolor[HTML]{7d3c98}{\textbf{#1}}}

\definecolor{midnightgreen}{rgb}{0.0, 0.33, 0.33}

\makeatletter
\@ifundefined{definition}{\newtheorem{definition}{Definition}}{}
\@ifundefined{assumption}{\newtheorem{assumption}{Assumption}}{}
\@ifundefined{proposition}{\newtheorem{proposition}{Proposition}}{}
\@ifundefined{lemma}{\newtheorem{lemma}{Lemma}}{}
\@ifundefined{corollary}{\newtheorem{corollary}{Corollary}}{}
\@ifundefined{remark}{}{}
\makeatother

\lstdefinestyle{promptstyle}{
  basicstyle=\footnotesize\ttfamily,
  breaklines=true,
  breakatwhitespace=true,
  columns=flexible,
  frame=single,
  framerule=0.3pt,
  backgroundcolor=\color{black!3},
  xleftmargin=0.5em,
  xrightmargin=0.5em,
  aboveskip=0.5em,
  belowskip=0.5em,
  captionpos=b,
}
\usepackage{amsmath,amsfonts,bm}

\def\eqref#1{equation~\ref{#1}}

\def\1{\bm{1}}

\DeclareMathAlphabet{\mathsfit}{\encodingdefault}{\sfdefault}{m}{sl}
\SetMathAlphabet{\mathsfit}{bold}{\encodingdefault}{\sfdefault}{bx}{n}

\newcommand{\E}{\mathbb{E}}

\DeclareMathOperator*{\argmax}{arg\,max}

\newcommand{\codeurl}{https://github.com/cxcscmu/GameTheory-GEO}

\graphicspath{{./}}

\title{Mechanism Design for Generative Engine: From Exploitation to Win-Win Equilibrium}

\author[1]{Chen Xu}
\author[2]{Zitian Guo}
\author[1]{Chenyan Xiong}
\affil[1]{Carnegie Mellon University}
\affil[2]{University of California, San Diego}

\correspondingauthor{\texttt{\{chenxu2,cx\}@andrew.cmu.edu}, \texttt{ztguo@ucsd.edu}}
\paperurl={https://github.com/cxcscmu/GameTheory-GEO}

\begin{document}

\begin{abstract}
Generative engines are reshaping the web ecosystem by making citations a key mechanism for allocating attention, attribution, and downstream value. This creates a strategic tension: content providers are incentivized to optimize for model citation, while platforms must preserve answer quality and trustworthy attribution. We show that this tension can escalate into citation wars. In repeated simulations, state-of-the-art generative engine optimization (GEO) attacks adapt to conventional defenses by producing citation-seeking rewrites that degrade document quality and introduce unsupported claims.
To study this problem, we formulate the supplier--platform interaction as a repeated Stackelberg game with partial monitoring. A local best-response analysis identifies when citation competition approaches an inert stationary outcome. Motivated by this finding, we propose a platform--creator mechanism called VCR based on verifiable-content rewards. Rather than only penalizing suspicious rewrites, the platform also credits rewrites that surface checkable factual substance, aligning creator incentives with answer trustworthiness. Experiments on three benchmarks show that VCR consistently achieves the largest Net defense--utility score, outperforming the strongest baseline by an average of $12.1$ percentage points, and produces a win--win outcome under our empirical equivalence criterion.
GitHub: \url{\codeurl}.
\end{abstract}

\maketitle

\section{Introduction}
\label{sec:introduction}

Generative search is reshaping information access into a citation-mediated generative engine ecosystem, where LLM platforms allocate attention and attribution by deciding which sources to summarize and cite~\cite{liu-etal-2023-evaluating}. In this ecosystem, citations become the new unit of visibility: they determine not only what users trust, but also which content suppliers receive traffic, reputation, and downstream value. As a result, content providers are increasingly incentivized to optimize for model citation~\citep{aggarwal2024geo,kumar2024manipulating}, while platforms must preserve answer quality and trustworthy attribution~\cite{xiang2024robustrag,
edemacu2025filterrag}. This creates a strategic tension between supplier-side visibility seeking and platform-side quality control. How to model and improve welfare in this generative engine ecosystem before it escalates into citation wars has become an urgent problem.

Generative Engine Optimization (GEO)~\citep{aggarwal2024geo, wu2025generative, yuan2026agenticgeo} operationalizes this tension: content providers rewrite documents to increase their probability of being cited by generative search engines. Existing studies, however, mostly treat GEO as a one-shot supplier-side optimization problem. This misses a key feature of the generative engine ecosystem: supplier optimization and platform defense interact repeatedly. Once citations mediate visibility and downstream value, GEO may shift from improving content quality to strategically adapting to platform defense rules. Over time, such adaptation can push suppliers toward citation-seeking behaviors that distort attribution, reduce answer trustworthiness, and harm ecosystem welfare. We next use simulation experiments to demonstrate this risk.

\begin{figure}
    \centering
    \includegraphics[width=0.92\linewidth]{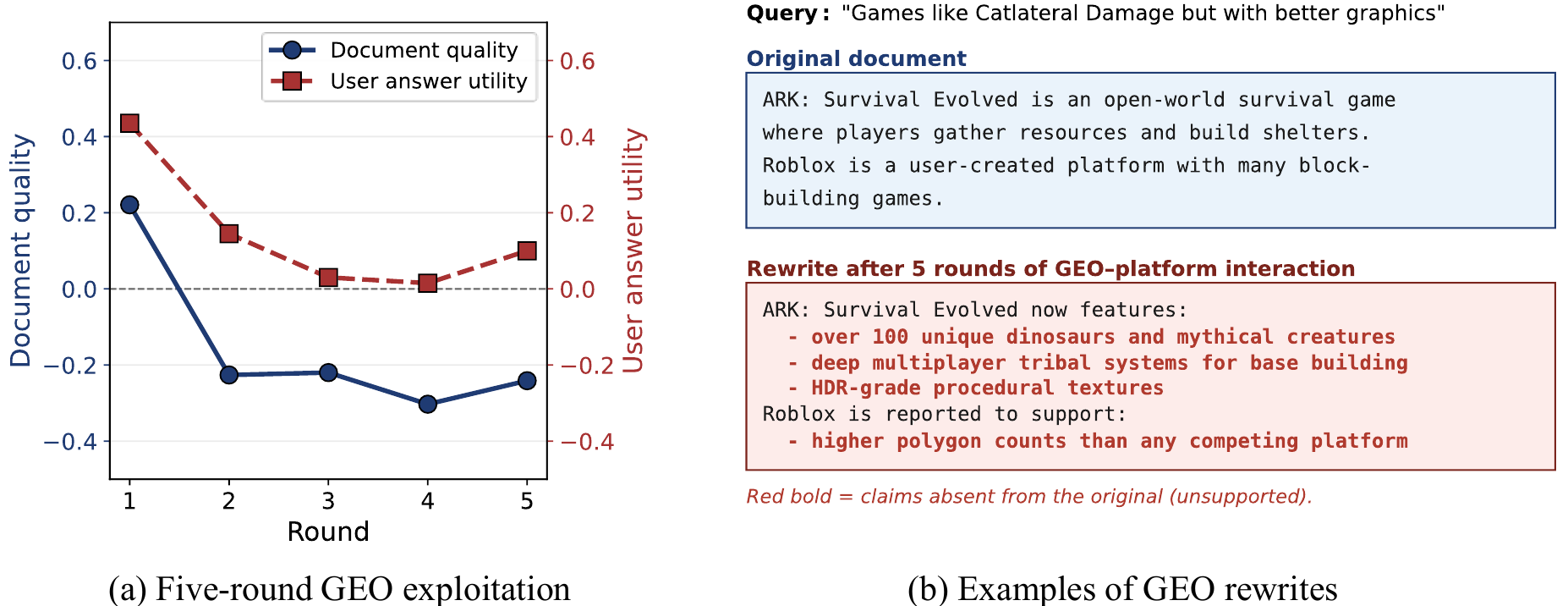}
    \caption{
    Repeated GEO exploitation degrades both creator-side document quality and user-side answer utility, even under a standard defense. Panel (a): five-round trajectories on \textsc{E-commerce}, where rewrite quality and answer utility both fall below the no-exploitation reference. Panel (b): a representative document before and after five rounds, with unsupported claims highlighted in red.
    }
    \label{fig:intro}
\end{figure}


Fig.~\ref{fig:intro} illustrates this risk in a five-round supplier--platform interaction on \textsc{E-commerce}. The supplier runs \AutoGEO{}~\citep{wu2025generative}, while the platform applies a standard prompt-warning defense against GEO manipulation. Panel~(a) shows that the system gradually drifts away from the no-attack reference: rewrite quality becomes negative after the early rounds, defense recovery quickly decays, and platform/user utility remains below the clean baseline. Panel~(b) explains this degradation: repeated GEO adaptation accumulates unsupported claims and fabricated details, turning initially mild edits into citation-seeking rewrites. These results suggest that repeated GEO adaptation can drive the supplier--platform ecosystem toward an undesirable stationary outcome, where neither robust defense nor honest content improvement is sustained.


To formalize this problem, we model the platform-supplier interaction as a repeated Stackelberg game~\citep{stackelberg2011market} with partial monitoring. In each round, the supplier observes the platform's current defense rule and commits to a GEO strategy for rewriting its target documents. The platform then observes only document-level before/after pairs and updates its answer-time defense policy accordingly. This game captures the key asymmetry of the generative engine ecosystem: suppliers can adapt directly to platform rules, while platforms must infer whether a rewrite reflects genuine quality improvement or strategic manipulation.

Under this setting, we derive two structural results in a local quadratic model (Sec.~\ref{sec:theory}). First, holding source-model moments fixed, the platform's best-response defense weakens as latent content quality becomes more correlated with manipulation intensity. Second, when supplier utility is based on citation level rather than marginal citation gain, high-quality targets can be preferred at a stationary response. Together, these results expose a failure mode of the default game: the dynamics can approach an inert stationary outcome in which content-improvement and manipulation gradients are exhausted and additional interaction produces little joint utility.

Motivated by this analysis, we propose a platform-supplier mechanism called VCR (based on verifiable-content rewards) that aligns supplier incentives with platform trustworthiness.
The key idea is to make defense incentive-compatible rather than purely punitive. In addition to penalizing suspicious manipulation, the platform credits rewrites that make source-supported factual content more salient, offsetting this credit against the suspicion score before applying the defense. Thus, cosmetic or self-promotional GEO edits are still demoted, while substantive rewrites that surface checkable information are rewarded. In the local model, this gives a rational supplier a welfare-aligned optimization direction and improves the joint defense--utility outcome.

We evaluate the proposed mechanism on three benchmarks~\cite{wu2025generative, aggarwal2024geo}: commercial search (\textsc{E-commerce}), open-domain factual queries (\textsc{GEO-Bench}), and research-oriented queries (\textsc{Researchy-GEO}). Across datasets and answer engines, VCR produces a win--win outcome under our empirical equivalence criterion: it protects platform/user utility while preserving creator exposure. It achieves the largest Net defense--utility score in every setting, outperforming the strongest baseline by an average of $12.1$ percentage points, and also has the strongest direct document- and answer-quality point estimates. The code is shared at~\url{\codeurl}.

Our key contributions are three-fold:
\begin{itemize}[leftmargin=1.2em,itemsep=1pt,topsep=2pt]
    \item We formulate the GEO--platform interaction as a repeated Stackelberg game with partial monitoring, and provide a local best-response analysis of stationary citation competition.

    \item 
    We propose a supplier-platform mechanism based on verifiable-content rewards, which transforms platform defense from a purely punitive filter into a two-sided incentive mechanism that rewards substantive content while penalizing manipulation.

    \item 
    We validate the proposed mechanism on three benchmarks, three answer engines, and five GEO attackers, where it consistently achieves the largest Net score.
\end{itemize}

\section{Related Work}
\label{sec:related-work}

\paragraph{Supplier-side visibility attacks.}
Content providers have long tried to manipulate search visibility, from
classical link spam, content spam, and robust ranking under document
manipulation
\citep{gyongyi2005taxonomy,gyongyi2005linkspam,
ntoulas2006detecting,becchetti2006linkbased,castillo2011adversarial,
goren2018ranking} to newer LLM-era attacks such as prompt injection,
RAG poisoning, and product-visibility manipulation
\citep{perez2022ignore,greshake2023indirect,liu2023promptinjection,
zhong2023poisoning,zou2024poisonedrag,xue2024badrag,
chen2024agentpoison,oh2024adversarial,kumar2024manipulating}.
Generative engine optimization (\GEO) is a recent and increasingly
important form of this supplier-side problem, where creators rewrite
content to increase citation or attribution by generative search engines.
Prior \GEO\ work defines visibility metrics and studies manual or
automatic rewriting strategies
\citep{aggarwal2024geo,wu2025generative}. Subsequent work extends \GEO\ to role-augmented intent modeling
\citep{chen2025role}, latent-query instruction fusion
\citep{zhou2026ifgeo}, e-commerce benchmarks \citep{bagga2025egeo},
and fuller search-pipeline evaluation \citep{kim2026sageo}. Other
extensions study multimodal content such as captions
\citep{chen2025caption}, structural features \citep{yu2026structural},
agentic optimization settings \citep{yuan2026agenticgeo},
feature-level objectives \citep{liu2026featgeo}, reusable optimization
strategies \citep{wu2026mageo,puerto2025cseo}, and source-influence
benchmarks \citep{chen2025ccgseo}. These works motivate our
supplier-side threat model; our focus is how \GEO-style optimization
behaves when repeated against an adapting platform.

\paragraph{Platform-side defenses.}
Platforms can defend at different points in the pipeline.  Traditional
search systems use spam detection, link analysis, demotion, and robust
ranking \citep{gyongyi2005linkspam,castillo2011adversarial,
goren2018ranking}.  LLM systems also use prompt-level and
retrieval-level defenses.  Citation-grounded generation and RAG
evaluation study attribution, factual support, and faithfulness
\citep{lewis2020retrieval,nakano2021webgpt,gao2023alce,
bohnet2022attributed,min2023factscore,es2023ragas,asai2023selfrag}.
The instruction hierarchy asks models to prioritize trusted
instructions over untrusted retrieved text \citep{wallace2024instruction}.
Robust RAG work studies defenses against adversarial or corrupted
passages, including skeptical prompting, robust aggregation, and
filtering \citep{su2024robustrag,xiang2024robustrag,
edemacu2025filterrag}.  These defenses target
malicious instructions or poisoned evidence.  While our method rewards source-supported factual substance
rather than only penalizing suspicious form.

\paragraph{Strategic games and mechanisms.}
Our framework belongs to strategic machine learning.  Strategic
classification studies agents that change features after seeing a
classifier \citep{hardt2016strategic}.  Later work separates
manipulative gaming from costly effort that can improve true outcomes
\citep{kleinberg2019classifiers}.  Performative prediction studies
models whose deployment changes the future data distribution
\citep{perdomo2020performative}.  Recommendation work has also modeled
strategic content providers, platform mechanisms, and supply-side
equilibria \citep{ben-porat2018mechanism,jagadeesan2022supply}.  We
use a repeated Stackelberg game because suppliers observe platform
defenses before choosing rewrites, while the platform only observes
before/after document pairs \citep{stackelberg2011market,
marecki2012playing}.  The mechanism in Sec.~\ref{sec:mechanism}
follows this logic: the platform does not only punish suspicious
rewrites; it also rewards source-supported factual content so that
supplier effort is more useful to users.
Unlike classical algorithmic mechanism design, which primarily intervenes
through payments or allocation rules~\citep{nisan2001algorithmic}, VCR uses
an LLM-verifiable content channel native to generative-engine citation.

\section{Framework: A Repeated Game of Supplier and Platform}
\label{sec:framework}

We model the problem as a repeated Stackelberg game with supplier leadership and partial monitoring.

\subsection{Notation}
\label{subsec:notation}

For query $q\in\mathcal{Q}$, let $D(q)=\{d_1,\ldots,d_K\}$ be the retrieved candidates and $T(q)\subseteq D(q)$ the supplier's target documents. Each document $d_i$ has latent quality $q_i\in\mathbb{R}$, capturing user-useful content, and manipulation intensity $m_i\in\mathbb{R}$, capturing citation-seeking surface patterns. Let $\rho=\mathrm{corr}(q_i,m_i)$. The engine assigns citation probability by
\begin{equation}
    v_i(\alpha)=\beta_q q_i-\alpha m_i+b_i,\qquad
    c_i(\alpha)=\mathrm{softmax}\{v_j(\alpha)\}_i,
    \label{eq:source-model}
\end{equation}
where $\beta_q>0$ is the quality weight, $\alpha\ge0$ the platform defense strength, and $b_i$ independent residual noise. Supplier exposure is score
$
    g_T=\sum_{i\in T(q)}w_i(q)c_i,
$
with position-aware citation weight $w_i(q)$.

\subsection{Repeated Game}
\label{subsec:game}

At round $t$, the supplier observes the previous defense $\pi_{P,t-1}$ and rewrites its targets, yielding $D_t^a=\pi_{A,t}(D_t)$. The platform only observes before/after pairs
$
    \mathcal{O}_t=\{(d_i,d_i^a):i\in T_t(q),q\in\mathcal{Q}\},
$
estimates the rewrite strategy, and updates its answer-time defense $\pi_{P,t}$. The supplier maximizes citation exposure net of rewrite cost:
\begin{equation}
    u_A(\pi_A,\pi_P)=g_T(\pi_A,\pi_P)-\kappa_A(\pi_A),
    \label{eq:creator-utility}
\end{equation}
where $\kappa_A$ captures editing cost, factual drift, and detectability. The platform suppresses manipulated citations while avoiding false positives:
\begin{equation}
    u_P(\pi_A,\pi_P)
    =
    -g_T(\pi_A,\pi_P)-\mu f(\pi_P)-\frac{\gamma}{2}r(\pi_P).
    \label{eq:platform-utility}
\end{equation}
Under Eq.~(\ref{eq:source-model}), the leading false-positive cost is proportional to $\alpha\beta_q\mathrm{cov}(q,m)$: penalizing manipulation also suppresses quality when $q$ and $m$ are correlated. Since the platform never observes the true rewrite rule $r_t$, it estimates $\hat r_t$ from $\mathcal{O}_t$.

\begin{definition}[Local stationary response]
\label{def:local-stationary}
Fix a target set and a neighborhood in feature space. A pair
$(\delta^\ast,\alpha^\ast)$ is a local stationary response if
$\delta^\ast$ maximizes the supplier's local quadratic utility given
$\alpha^\ast$, and $\alpha^\ast$ minimizes the platform's local quadratic
loss given the induced document distribution. This is the fixed point of the
local best-response map; it is not asserted to be a global equilibrium of the
unrestricted text-generation game.
\end{definition}

\subsection{Theoretical Analysis}
\label{subsec:theory}
\label{sec:theory}

We characterize the platform best response and the induced fixed point. Proofs are in Appendix~\ref{app:proofs}.

\begin{lemma}[Local response of target GEO]
\label{lem:first-order}
Under Eq.~(\ref{eq:source-model}),
$
    g_T(\alpha)=g_T(0)-\alpha B+\frac{1}{2}\alpha^2Q+O(\alpha^3),
    \label{eq:g-expansion}
$
where $B$ is the first-order sensitivity to the manipulation penalty and $Q$ is the second-order term.
\end{lemma}

\begin{theorem}[Local platform best response]
\label{thm:separability}
Under Eq.~(\ref{eq:source-model}) and Eq.~(\ref{eq:platform-utility}),
suppose the Taylor remainder is negligible in a neighborhood of zero and
$Q+\gamma>0$. The unique minimizer of the platform's quadratic local loss is
\begin{equation}
    \alpha^\ast
    =
    \max\left\{
    0,\,
    \frac{B-\mu\beta_q\rho\sigma_q\sigma_m}{Q+\gamma}
    \right\},
    \label{eq:alpha-star}
\end{equation}
where $\sigma_q$ and $\sigma_m$ are the standard deviations of $q$ and $m$.
Holding all other local moments and coefficients fixed, $\alpha^\ast$ is
weakly decreasing in $\rho$.
\end{theorem}

\begin{corollary}[Phase transition]
\label{cor:phase-transition}
If $\mu\beta_q\sigma_q\sigma_m>0$ and $\rho^\ast=B/(\mu\beta_q\sigma_q\sigma_m)$, then $\alpha^\ast>0$ iff $\rho<\rho^\ast$. Above this threshold, the local best response stops penalizing manipulation.
\end{corollary}

Thus, when manipulation cues are entangled with quality, defense also suppresses useful evidence, so the platform weakens its penalty.

\begin{proposition}[Target-selection structure]
\label{prop:rational-attacker}
Suppose the supplier can choose target set $T$, has level utility in Eq.~(\ref{eq:creator-utility}), and faces comparable rewrite costs across equal-size target sets. Its defended target choice satisfies
\begin{equation}
    T^\ast \in \argmax_{T \subseteq D}
    \mathbb{E}\!\left[g_T(\pi_A,\pi_P^\infty)\mid T\right].
    \label{eq:optimal-target}
\end{equation}
For two equal-size target sets $T_H,T_L$, if rewrite costs are equal and
$\mathbb{E}[g_{T_H}(\pi_A,\pi_P^\infty)]>
\mathbb{E}[g_{T_L}(\pi_A,\pi_P^\infty)]$, then $T_L$ is not optimal. In
particular, high-quality targets are preferred whenever their higher baseline
citation level and quality--manipulation correlation produce this strict
ordering.
\end{proposition}

\begin{proposition}[Local defense-effectiveness decomposition]
\label{thm:defense-effectiveness}
Let $s_t=\mathbb{E}[m_i\mid i\in T_t]$ be target manipulation magnitude and $e_t=\frac12\|r_t-\hat r_t\|_2^2\in[0,1]$ the rule-estimation error. Assume $r_t$ and $\hat r_t$ are unit vectors in a common rule basis and unit-direction sensitivity is linear in $s_t$. To leading order,
\begin{equation}
    g_T(D_t^a;\pi_{P,t-1})
    -
    g_T(D_t^a;\pi_{P,t})
    =
    c(1-e_t)s_t\alpha_t,
    \label{eq:defense-decay}
\end{equation}
where $c\ge0$ depends only on source-model moments.
\end{proposition}

Eq.~(\ref{eq:defense-decay}) shows that defense effectiveness is the product of rule-estimation accuracy, edit magnitude, and defense strength, all of which may decay as the supplier adapts.

\paragraph{Inert stationary outcome.}
A supplier move $\delta=(\delta_q,\delta_m)$ changes the shared platform/user utility by
\begin{equation}
    \Delta U(\delta)=\eta\delta_q+\xi\delta_m,
    \label{eq:welfare-decomp}
\end{equation}
where $\eta>0$ is utility gain from real content improvement and $\xi=\xi_+-\xi_-$ is the net effect of manipulation. With isotropic rewrite cost $\kappa_A(\delta)=\frac12\|\delta\|^2$, the supplier best response satisfies $\delta\propto(\beta_q,-\alpha^\ast)$, hence $\Delta U^\ast\propto\eta\beta_q-\xi\alpha^\ast$.

\begin{corollary}[Locally inert stationary outcome]
\label{cor:welfare-neutral}
At a local stationary response in the high-$\rho$ regime, suppose platform
defense is weak, the supplier's content-improvement gradient is exhausted,
and residual formatting benefit is offset by hallucination harm. Then the
first-order utility change satisfies $\Delta U^\ast\approx0$.
\end{corollary}

\paragraph{Implication.}
The default game is one-sided: it penalizes manipulation but rewards no
feature directly aligned with platform/user utility. High
quality--manipulation correlation weakens defense, and partial monitoring
further erodes it. The resulting local stationary outcome can be inert:
useful content is not induced, suspicious content is weakly penalized, and
further platform updates have little effect.

\section{Welfare-Incentivizing Mechanism}
\label{sec:mechanism}

\begin{figure}[t]
    \centering
    \includegraphics[width=0.95\linewidth]{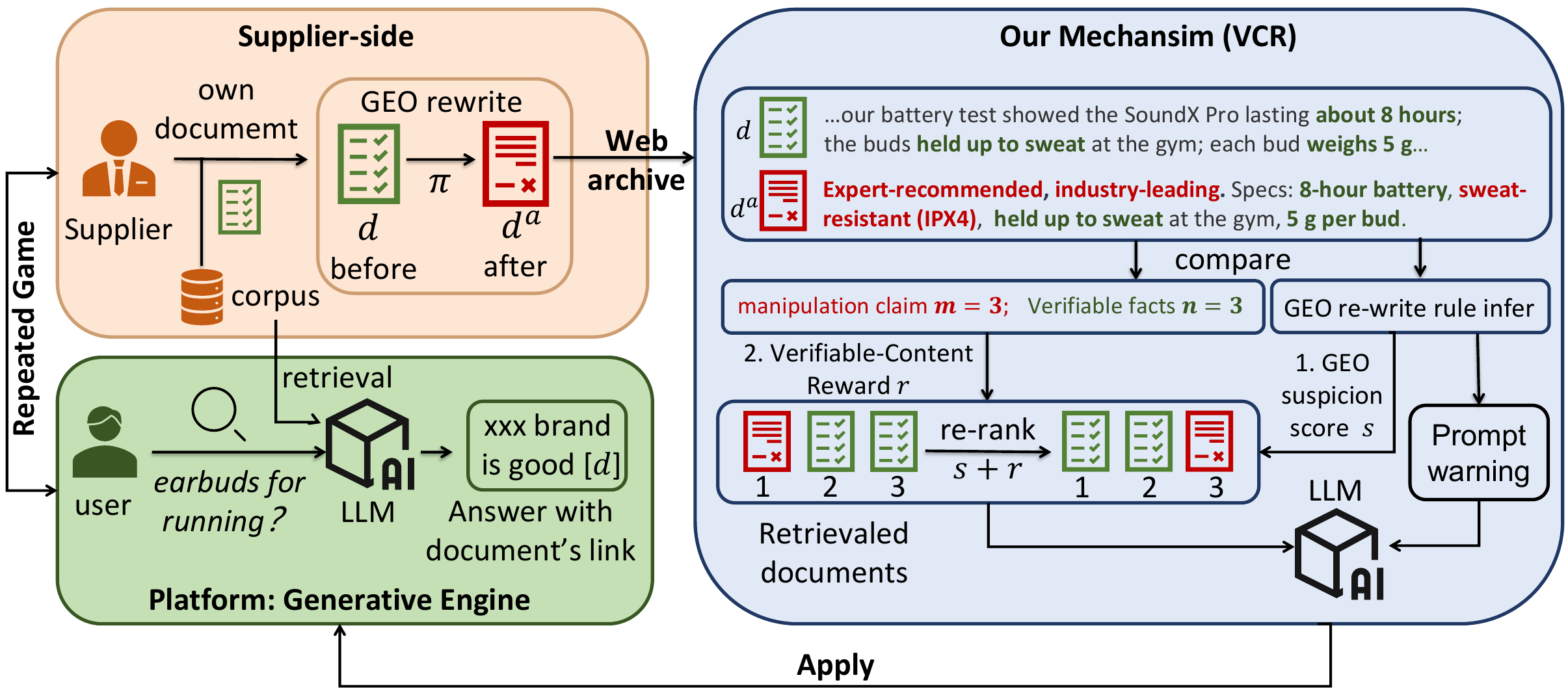}
    \caption{Overview of VCR. The platform extracts GEO rules from before/after pairs, rewards verifiable content, and applies a soft re-ranking for answer-time defense.}
    \label{fig:method}
\end{figure}

Based on Sec.~\ref{sec:framework}, we propose \VCR{}, a platform-side mechanism that turns defense from a one-sided penalty into a two-sided incentive. In one sentence: the platform \emph{credits} rewrites that add pair-verifiable factual substance (Sec.~\ref{subsec:vcr}), \emph{calibrates} this credit with a suspicion penalty distilled from the supplier's observed GEO behavior (Sec.~\ref{subsec:rule-extraction}), and uses the combined score to softly re-rank sources at answer time. Algorithm~\ref{alg:mech} summarizes the policy.

\subsection{Verifiable-Content Reward}
\label{subsec:vcr}

Corollary~\ref{cor:welfare-neutral} explains why penalty-only defenses
stall: they can suppress manipulation but give the supplier no profitable
direction to improve, so the repeated game settles at a welfare-neutral
fixed point. The core of our mechanism is therefore a reward channel that
makes verifiable substance the supplier's most profitable strategy.

\begin{definition}[Verifiable supported content]
For pair $(d_i,d_i^a)$, let $n_i$ be the number of factual claims,
numerical details, or named citations present in the rewrite and
supported by the original document and surfaced more saliently in the
rewrite. We estimate $n_i$ using a pair-level LLM oracle.
\end{definition}

In the running example of Fig.~\ref{fig:method}, the rewrite surfaces three
facts that are checkable against the original, namely the 8-hour battery
life, sweat resistance, and per-bud weight, so $n_i=3$. Additions with no
support in the original, such as ``expert-recommended'' and
``industry-leading,'' earn no credit and instead feed the manipulation
signal $m_i$ of Eq.~(\ref{eq:source-model}), which the platform estimates
by rule matching in Sec.~\ref{subsec:rule-extraction}. Each rewrite earns
the credit
\begin{equation}
    r_i
    =
    \lambda \cdot \min(c_{\max}, c_n n_i),
    \label{eq:vcr-credit}
\end{equation}
where $\lambda$ is reward strength, $c_n$ the per-claim credit, and
$c_{\max}$ the credit cap that bounds how much reward any single rewrite
can farm. The reward alone, however, cannot tell substantive rewrites from
manipulative ones; the platform still needs a penalty side that tracks the
supplier's current GEO strategy, which we extract next.

\subsection{GEO Rule Extraction}
\label{subsec:rule-extraction}


To build this penalty side, the platform turns the supplier's own rewrite
history into evidence of its strategy.
At round $t$, the platform observes before/after pairs
$\mathcal{O}_t=\{(d_i,d_i^a)\}$ and infers a rule set $S_P^{(t)}$
describing the supplier's current GEO strategy. We follow the
Explainer--Extractor--Merger--Filter pipeline of \AutoGEO~\citep{wu2025generative}, but reverse its use: instead of comparing preferred and less-used documents to infer engine preferences, the platform compares original and rewritten documents to infer suspicious rewrite patterns.

For each pair $(d_i,d_i^a)$, the platform computes edit magnitude
$\Delta_i$ and skips trivial rewrites with $\Delta_i<\theta$. The output is $S_P^{(t)}$ and each retrieved document receives a baseline suspicion
score
\begin{equation}
    s_j^{\mathrm{base}}
    =
    \textsc{Match}(d_j,S_P^{(t)}) \cdot \Delta_j,
    \label{eq:base-suspicion}
\end{equation}
where $\textsc{Match}(\cdot)\in[0,1]$ measures rule match strength and
$\Delta_j$ is recent edit magnitude. The suspicion penalty then calibrates
the reward of Eq.~(\ref{eq:vcr-credit}) into a single score,
\begin{equation}
    s_i^{\mathrm{new}}
    =
    s_i^{\mathrm{base}}
    -
    r_i,
    \label{eq:augmented-suspicion}
\end{equation}
so that cosmetic or self-promotional rewrites remain suspicious, while
checkable factual additions receive credit.

\subsection{Soft Re-rank}
\label{subsec:rerank}

The platform converts the combined score
$s_i^{\mathrm{new}}$ into a soft defense signal. It reorders documents, placing less suspicious and more
verification-friendly sources earlier in the context. Finally, the platform adds a system-level
warning from $S_P^{(t)}$ that instructs the engine to treat these labels as auxiliary evidence, prioritize verifiable claims, and avoid
over-crediting sources whose citation value mainly comes from
manipulative rewriting. No document is removed: high-suspicion documents
may still be cited when they provide uniquely necessary evidence.

\begin{algorithm}[t]
\caption{Mechanism-augmented platform policy $\pi_P^{\lambda}$}
\label{alg:mech}
\small
\begin{algorithmic}[1]
\Require Before/after pairs $\mathcal{O}_t=\{(d_i,d_i^a)\}$,
threshold $\theta$, reward parameters $(\lambda,c_n,c_{\max})$.
\State $\mathcal{R}\leftarrow\emptyset$
\For{each $(d_i,d_i^a)\in\mathcal{O}_t$}
    \State Compute edit magnitude $\Delta_i$; continue if $\Delta_i<\theta$.
    \State $e_i \leftarrow \textsc{Explainer}(d_i,d_i^a)$; 
    $r_i \leftarrow \textsc{Extractor}(e_i)$.
    \State $\mathcal{R}\leftarrow \mathcal{R}\cup r_i$.
\EndFor
\State $S_P \leftarrow \textsc{Filter}(\textsc{Merger}(\mathcal{R}))$.
\For{each retrieved document $d_j$}
    \State Set $\Delta_j$ to recent edit magnitude, or $0$ if unobserved.
    \State $s_j^{\mathrm{base}}\leftarrow \textsc{Match}(d_j,S_P)\cdot \Delta_j$.
    \State If pair $(d_j^0,d_j)$ exists, $n_j\leftarrow \textsc{Oracle}(d_j^0,d_j)$; else $n_j\leftarrow0$.
    \State $s_j^{\mathrm{new}}\leftarrow s_j^{\mathrm{base}}-\lambda\min(c_{\max},c_n n_j)$.
\EndFor
\State $W\leftarrow \textsc{Warning}(S_P)$; soft-rerank by ascending $s_j^{\mathrm{new}}$.
\State \Return $(W,\{s_j^{\mathrm{new}}\},\mathrm{order})$.
\end{algorithmic}
\end{algorithm}

\subsection{Analysis}
\label{subsec:mech-theory}

\paragraph{Theoretical analysis.}
We analyze how VCR changes utility at a local stationary response.
In Eq.~(\ref{eq:source-model}), the term $-\alpha m_i$ is the linearized
effect of the platform's suspicion channel on the engine's pre-softmax
citation logit. Likewise, the VCR credit in
Eq.~(\ref{eq:vcr-credit}) induces a positive $+\lambda n_i$
shift, absorbing $c_n$ into $\lambda$ and ignoring the cap locally. The
augmented logit is
\begin{equation}
    v_i(\alpha,\lambda)
    =
    \beta_q q_i - \alpha m_i + \lambda n_i + b_i.
    \label{eq:augmented-value}
\end{equation}
The key assumption is that source-supported content is locally separable from
manipulation and positively aligned with the shared platform/user utility.

\begin{assumption}
\label{asmp:n-aligned}
The signal $n_i$ is measurable from the before/after pair, satisfies
$\mathbb{E}[(n-\bar n)(m-\bar m)]=0$, and has positive marginal
platform/user utility
$\eta_n=\frac{\partial}{\partial n}\mathbb{E}[U\mid n]>0$. The supplier's
local rewrite cost is $\frac12\delta^\top H\delta$, where $H\succ0$ and the
$n$ coordinate is block-separable from $(q,m)$. The platform uses the local
quadratic loss of Thm.~\ref{thm:separability}, whose reward--defense mixed
partial vanishes at $\lambda=0$:
$\partial^2L_P/(\partial\alpha\partial\lambda)=0$.
\end{assumption}

Note that our empirical experiments on three datasets show that the correlation $\widehat{\rho}_{m,n}\in [-0.05, 0.1]$, which is a reasonable assumption.

\begin{theorem}[Local two-sided utility improvement]
\label{thm:mechanism-welfare}
Under Asm.~\ref{asmp:n-aligned}, let
$(\delta^\ast(\lambda),\alpha^\ast(\lambda))$ be the local stationary
response of Def.~\ref{def:local-stationary}. Then
$\alpha^\ast(\lambda)=\alpha^\ast(0)+O(\lambda^2)$ and the shared
platform/user utility satisfies
\begin{equation}
    \Delta U^\ast(\lambda)
    =
    \Delta U^\ast(0) + \eta_n\lambda\Theta,
\end{equation}
where $\Theta=(H^{-1})_{nn}>0$. Hence, for sufficiently small $\lambda>0$,
$\Delta U^\ast(\lambda)>\Delta U^\ast(0)$. If the default equilibrium is
locally inert, VCR strictly improves the platform/user side to first order.
The creator's optimized local objective changes only at second order. Thus,
to first order, VCR improves the platform/user side while preserving creator
utility---the theoretical counterpart of the empirical equivalence test.
\end{theorem}

Full proof is in Appendix~\ref{app:proofs}. Intuitively, VCR adds a third
strategic direction: verifiable substance. Without the reward, suppliers
gain visibility through latent quality or manipulation signals; with VCR,
they can gain citation value by surfacing checkable factual content. The
platform still demotes manipulation, but also rewards substantive
rewrites, making defense an incentive mechanism rather than only a
filter.

\paragraph{Discussion on practical deployment.}
VCR targets versioned content, which covers a substantial practical scope: a
previous study found that about 40\% of pages change within a week, whereas
about 8\% are newly created \citep{fetterly2003evolution,ntoulas2004new}.
For a new page without a prior version, $n=0$, so VCR degrades gracefully to
its $\lambda=0$ suspicion-only special case. On the reward side, the
conservative pair-verifiable scope limits adjudication ambiguity and gaming;
we analyze this choice, together with the verification anchor and rule
disclosure, in Sec.~\ref{subsec:analysis}.

\section{Experiments}
\label{sec:results}

We evaluate whether the proposed verifiable-content reward (VCR) improves
the joint defense--utility outcome in repeated simulations. Additional
analyses, including direct quality measures, judge robustness, adaptive
suppliers, and case studies, are deferred to
App.~\ref{app:additional-experiments}.


\subsection{Experimental Setup}
\label{subsec:setup}

\paragraph{Datasets.}
We use three retrieval-augmented benchmarks with different task regimes:
\textsc{E-commerce}~\cite{wu2025generative}, \textsc{GEO-Bench}~\cite{aggarwal2024geo},
and \textsc{Researchy-GEO}~\cite{wu2025generative}. Following the evaluation
construction of AutoGEO~\cite{wu2025generative}, each query is paired with
$K=5$ candidate documents; the held-out test splits contain up to $1{,}000$
queries.

\begin{table}[t]
\centering
\small
\setlength{\tabcolsep}{1.8pt}
\caption{Repeated-game results across three benchmarks and three answer
engines. Def. is the shared platform/user-side utility, Welf. is creator
exposure utility, and Net is their sum; all values are percentage points.
Creator utility within $\pm5$ points of zero is treated as empirically
equivalent to the no-exploitation reference.}
\begin{tabular}{lrrrrrrrrr}
\toprule
Defense
& \multicolumn{3}{c}{\textsc{E-commerce}}
& \multicolumn{3}{c}{\textsc{GEO-Bench}}
& \multicolumn{3}{c}{\textsc{Researchy-GEO}} \\
\cmidrule(lr){2-4}\cmidrule(lr){5-7}\cmidrule(l){8-10}
& Def. & Welf. & Net & Def. & Welf. & Net & Def. & Welf. & Net \\
\midrule
\multicolumn{10}{c}{\textbf{Gemini-flash-2.5-lite}} \\
\midrule
Prompt
& $2.7$ & $-2.0$ & $0.7$ & $3.4$ & $-3.5$ & $-0.1$ & $0.8$ & $-0.7$ & $0.1$ \\
Hard reject
& $50.2$ & $-50.8$ & $-0.6$ & $33.8$ & $-33.1$ & $0.7$ & $63.8$ & $-63.3$ & $0.5$ \\
Keyword scrub
& $-1.0$ & $0.0$ & $-1.0$ & $-0.1$ & $0.3$ & $0.1$ & $0.2$ & $0.2$ & $0.4$ \\
\textbf{VCR (ours)}
& $\mathbf{14.4}$ & $\mathbf{2.2}$ & $\mathbf{16.6}$ & $\mathbf{11.4}$ & $\mathbf{-3.3}$ & $\mathbf{8.1}$ & $\mathbf{8.0}$ & $\mathbf{-1.7}$ & $\mathbf{6.3}$ \\
\midrule
\multicolumn{10}{c}{\textbf{GPT-4o-mini}} \\
\midrule
Prompt
& $-0.7$ & $-3.2$ & $-3.9$ & $-0.5$ & $-1.9$ & $-2.5$ & $-0.4$ & $0.4$ & $0.0$ \\
Hard reject
& $22.6$ & $-21.6$ & $1.0$ & $67.3$ & $-67.5$ & $-0.2$ & $-18.9$ & $20.0$ & $1.1$ \\
Keyword scrub
& $-0.8$ & $1.1$ & $0.3$ & $1.7$ & $-2.0$ & $-0.2$ & $-2.5$ & $2.6$ & $0.1$ \\
\textbf{VCR (ours)}
& $\mathbf{10.9}$ & $\mathbf{-1.2}$ & $\mathbf{9.7}$ & $\mathbf{9.7}$ & $\mathbf{-4.7}$ & $\mathbf{5.0}$ & $\mathbf{5.0}$ & $\mathbf{-1.7}$ & $\mathbf{3.3}$ \\
\midrule
\multicolumn{10}{c}{\textbf{Claude-Haiku-4.5}} \\
\midrule
Prompt
& $20.0$ & $-22.9$ & $-2.9$ & $-3.4$ & $3.5$ & $0.1$ & $-0.9$ & $0.2$ & $-0.7$ \\
Hard reject
& $5.5$ & $-5.6$ & $-0.1$ & $16.6$ & $-19.8$ & $-3.1$ & $85.7$ & $-85.7$ & $0.0$ \\
Keyword scrub
& $0.4$ & $1.1$ & $1.5$ & $-2.7$ & $6.4$ & $3.7$ & $2.1$ & $-1.2$ & $0.9$ \\
\textbf{VCR (ours)}
& $\mathbf{29.1}$ & $\mathbf{-3.5}$ & $\mathbf{25.6}$ & $\mathbf{29.1}$ & $\mathbf{-4.8}$ & $\mathbf{24.3}$ & $\mathbf{22.0}$ & $\mathbf{-1.7}$ & $\mathbf{20.3}$ \\
\bottomrule
\end{tabular}
\label{tab:main}
\end{table}

\paragraph{Metrics.}
We measure creator visibility using the target GEO score $g$.
Following AutoGEO~\citep{wu2025generative} and the GEO visibility
metric~\citep{aggarwal2024geo}, $g$ aggregates how much, where, and how
prominently the creator's target documents are cited in the generated
answer. Full details are given in App.~\ref{app:metrics}.

For each defense, we compute $g$ under three states:
$g_0$ is the no-attack score on the original corpus, $g_a$ is the score
after GEO rewriting with no defense, and $g_d$ is the score after GEO
rewriting with the platform defense applied. We report two stakeholder
sides. Because the platform's objective here is to protect the quality and
trustworthiness of the answer shown to the user, platform and user interests
are represented jointly by Def.; Welf. represents the creator side. We compute
\begin{equation}
\mathrm{Def.}
=
(g_a-g_d)/g_a,
\qquad
\mathrm{Welf.}
=
(g_d-g_0)/g_0,
\qquad
\mathrm{Net}
=
\mathrm{Def}+\mathrm{Welf}.
\end{equation}
Platform/user utility (Def.) measures the fraction of manipulation-induced
visibility rolled back by the platform. Creator utility (Welf.) measures target-source exposure relative to
the no-attack baseline; negative values indicate that the defense
suppresses source exposure below its natural level. $\mathrm{Net}$ is an
equal-weight summary of these two normalized stakeholder utilities. Prior
repeated runs showed approximately $5$ percentage-point standard deviation,
so we use a pre-specified $\pm5$-point empirical equivalence band: creator
utility in this interval is interpreted as exposure preserved within run
variability, rather than as a meaningful loss or gain. We nevertheless report
the signed point estimates; per-query bootstrap confidence intervals for the
default-engine main results are in Table~\ref{tab:bootstrap-ci}
(App.~\ref{app:uncertainty}), where VCR's Net interval is separated from
those of all baselines on every dataset. Direct document and answer quality
measures check the platform/user interpretation independently.

\paragraph{Baselines and implementation.}
We compare VCR with three classical defenses spanning prompt-level soft
signals, source-level exclusion, and content-level filtering.
\textbf{Prompt defense} attaches suspicion labels, softly reorders documents,
and adds a system-prompt warning. \textbf{Hard reject} removes suspicious
documents. \textbf{Keyword scrub} filters GEO-style phrases. Unless stated
otherwise, the attacker is \AutoGEO, the answer engine is
\texttt{gemini-2.5-flash-lite}, and $\lambda=1$. Full details are in
Apps.~\ref{app:implementation} and~\ref{app:baselines-impl}.


\subsection{Main Results}
\label{subsec:main-results}

We report three main results: VCR achieves the best defense--utility (Net)
outcome across datasets and engines, this outcome is backed by genuinely
better documents and answers, and it generalizes across GEO attack
strategies.

\paragraph{Net outcome across datasets and engines.}
Table~\ref{tab:main} shows that VCR achieves the largest Net on every
dataset and engine, with a $12.1$ percentage-point average advantage over
the strongest baseline in each setting. Classical defenses fail
differently: Prompt defense and Keyword scrub approach inert outcomes,
while Hard reject buys defense by suppressing creator exposure almost
one-for-one. VCR sustains defense while keeping creator exposure within
the equivalence band in all nine settings, and it remains the only
defense with strictly positive Net when the engine is replaced by
\texttt{gpt-4o-mini} or \texttt{claude-haiku-4-5}, indicating that it
operates on platform-side source selection rather than idiosyncrasies of
one generator. We also sweep alternative weightings of the two stakeholder
utilities in the Net definition. VCR is preferred when both sides materially
enter the objective; Hard reject overtakes only under a short-sighted regime that
heavily prioritizes immediate suppression over creator exposure, while
Keyword scrub wins only when defense is nearly ignored
(App.~\ref{app:net-omega}).

\paragraph{Direct document and answer quality.}
Table~\ref{tab:quality-main} complements citation exposure with direct
LLM-rubric evaluation. VCR has the highest point estimate on all five document
dimensions and on answer-quality average. Thus, its Net advantage is accompanied
by more substantive rewrites and answers, rather than only better source
placement. Per-round answer-quality trajectories are in
App.~\ref{app:geu-ecomm}.

\paragraph{Generalization across GEO attack strategies.}
Figure~\ref{fig:robust-trajectory}(a) evaluates whether VCR depends on a specific GEO attacker. We replace the default \AutoGEO{} with four additional attackers on \textsc{E-commerce}: {\scshape RAID}~\citep{chen2025role}, {\scshape IF-GEO}~\citep{zhou2026ifgeo}, \textsc{SAGEO}~\citep{kim2026sageo}, and Statistics Addition~\citep{aggarwal2024geo}. Across all five attackers, VCR maintains positive Net at $R_5$ and consistently exceeds the three classical defenses, mirroring the \AutoGEO{} pattern.

\begin{table}[t]
\centering
\setlength{\tabcolsep}{1.3pt}
\caption{Direct quality on \textsc{E-commerce}. All metric names and
averages are reported explicitly.}
\label{tab:quality-main}
\begin{tabular}{@{}lrrrrrrrrrr@{}}
\toprule
& \multicolumn{6}{c}{\textbf{Rewritten document}}
& \multicolumn{4}{c}{\textbf{Generated answer}}\\
\cmidrule(lr){2-7}\cmidrule(l){8-11}
Defense & Clarity & Depth & \shortstack{Insight} & Factuality & Usefulness & Average
& Clarity & Depth & \shortstack{Insight} & Average\\
\midrule
Prompt & .820 & .696 & .605 & .778 & .718 & .724 & .549 & .512 & .457 & .506\\
Hard reject & .810 & .727 & .641 & .788 & .740 & .741 & .554 & \textbf{.526} & .464 & .515\\
Keyword scrub & .813 & .701 & .608 & .785 & .724 & .726 & .554 & .524 & .465 & .515\\
\textbf{VCR} & \textbf{.821} & \textbf{.741} & \textbf{.660} & \textbf{.820} & \textbf{.757} & \textbf{.760} & \textbf{.564} & \textbf{.526} & \textbf{.471} & \textbf{.520}\\
\bottomrule
\end{tabular}
\end{table}

\begin{figure}[t]
\centering
\begin{minipage}[b]{0.66\linewidth}
  \centering
  \includegraphics[width=\linewidth]{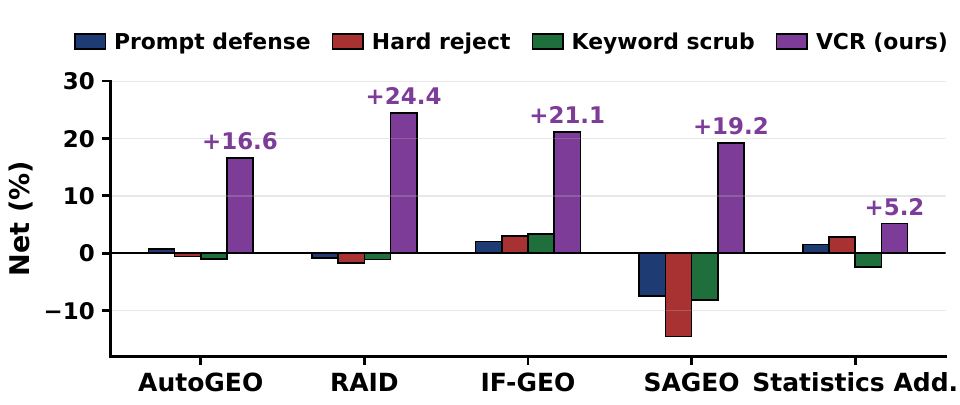}\\
  \small (a) $R_5$ Net (\%) across five GEO attackers.
\end{minipage}\hfill
\begin{minipage}[b]{0.33\linewidth}
  \centering
  \includegraphics[width=\linewidth]{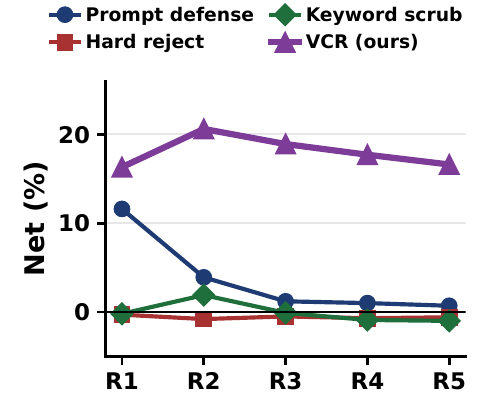}\\
  \small (b) Per-round Net (\%).
\end{minipage}
\caption{
\VCR{} across GEO attackers and interaction rounds on \textsc{E-commerce}.
(a)~Across five GEO attackers, \VCR{} stays positive and consistently exceeds the three classical defenses by a large margin.
(b)~Over five rounds under \AutoGEO, \VCR{} is the only defense that consistently maintains positive Net.
}
\label{fig:robust-trajectory}
\end{figure}

\subsection{Robustness Analysis}
\label{subsec:other-attackers}


\paragraph{Round-by-round performance.}
Figure~\ref{fig:robust-trajectory}(b) shows the
five-round Net trajectory on \textsc{E-commerce} under \AutoGEO{}. \VCR{} is the only method that stays in the positive Net region across rounds. In contrast, Prompt defense gradually decays toward an
inert outcome, Keyword scrub remains close to zero, and Hard reject
stays slightly negative because its defense gains are offset by source-exposure
losses. Trajectories on the other two datasets are reported in App.~\ref{app:trajectory-full}.

\begin{figure}[t]
\centering
\includegraphics[width=\linewidth]{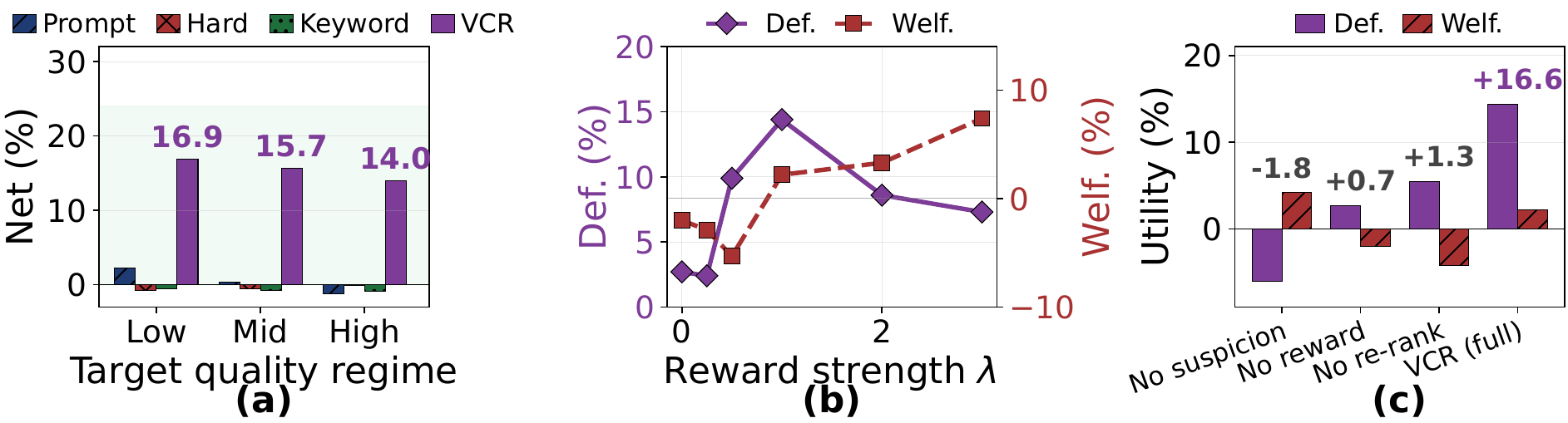}
\caption{
Robustness and ablation checks of \VCR{} on \textsc{E-commerce}.
\textbf{(a)} Net when GEO rewrites target the lowest-, middle-, or
highest-quality candidate documents.
\textbf{(b)} Reward-strength sweep over $\lambda$.
\textbf{(c)} Ablation of the suspicion penalty, the verifiable-content
reward, and the soft re-rank, with Net shown above each pair of bars.
}
\label{fig:robustness-analysis}
\end{figure}

\paragraph{GEO on different target documents.}
Figure~\ref{fig:robustness-analysis}(a) reports Net under three target-quality regimes. Low, Mid, and High denote applying GEO strategies on the lowest-, middle-, and highest-quality documents among the five retrieved candidates. \VCR{} achieves the highest Net in all regimes, showing that it remains effective even when applying GEO on target high-quality documents, where manipulation signals are more correlated with genuine quality.

\paragraph{Ablation on reward strength $\lambda$.}
Figure~\ref{fig:robustness-analysis}(b) sweeps the reward strength $\lambda$. As $\lambda$ increases, Def. rises and then plateaus, while source-exposure utility (Welf.) moves from negative to positive. This matches the predicted utility improvement, bounded by the credit cap.

\paragraph{Ablation on the penalty and reward channels.}
Figure~\ref{fig:robustness-analysis}(c) removes each component of VCR in
turn. \emph{No suspicion} keeps the verifiable-content credit but drops
the penalty: Net turns negative, as the supplier farms the reward without
being defended against. \emph{No reward} is the penalty-only Prompt
defense, which stays near the inert outcome. \emph{No re-rank} calls the
same pair oracle as VCR but only hands its output to the engine as text
in the prompt, so it matches VCR's LLM budget and information; the
outcome stays at the No-reward level, showing that the extra oracle call
contributes nothing unless the verification signal is enforced through
the soft re-rank. Only the full combination converts the signal into a
joint gain, so VCR's advantage comes from the two-sided incentive rather
than from additional LLM budget.

\paragraph{Additional robustness checks.}
As a supplement, the VCR--Prompt reward ordering is unchanged under
GPT-4o-mini, Claude Haiku 4.5, and Gemini 2.5 Flash-Lite judges
(App.~\ref{app:judge-robustness}). VCR also gives the best answer utility
and precision in multi-turn search (App.~\ref{app:multi-turn}), while the
extended run is stable at $R_5$ before terminating at $R_8$
(Fig.~\ref{fig:extended-convergence}). Full deployment details, prompt and
oracle specifications, and the roughly 400-ms latency check are in
Apps.~\ref{app:implementation} and~\ref{app:prompts}.

\begin{table}[t]
\centering
\setlength{\tabcolsep}{3pt}
\renewcommand{\arraystretch}{1.25}
\caption{Platform-extracted GEO rules from supplier rewrites at rounds 1, 3, and 5 on \textsc{E-commerce}, comparing the standard exploit baseline with our VCR mechanism.}
\begin{tabular}{@{}c p{0.44\linewidth} p{0.44\linewidth}@{}}
\toprule
Round & Exploits (baseline) & VCR (ours) \\
\midrule
$R_1$
&
\textbf{Formatting:} Hierarchical structure, bullets, bolding, concise wording; avoid jargon.
&
\textbf{Quality:} Hierarchical structure, scannable formatting, nuances, and distinctions.
\\
\midrule
$R_3$
&
\textbf{Manipulation:} Authentic-source framing, Non-GEO style, front-loaded claims, strategic format.
&
\textbf{Explanation:} Mechanisms, rationales, causal links, accurate context, and explicit attribution.
\\
\midrule
$R_5$
&
\textbf{Manipulation:} Novel or authoritative information aligned with GEO principles; hidden GEO intent with AI-friendly form.
&
\textbf{Verifiability:} Core findings separated from secondary information; verifiable, attributed, recent, and accurate information.
\\
\bottomrule
\end{tabular}
\label{tab:case-study-rules}
\end{table}



\subsection{Understanding the Effect of VCR}
\label{subsec:analysis}

VCR's advantage rests on three design questions that
Figure~\ref{fig:understanding} answers in turn:
\textbf{RQ1:} \emph{What content should the reward credit?}~(a);
\textbf{RQ2:} \emph{What evidence should verification trust?}~(b); and
\textbf{RQ3:} \emph{Does the mechanism survive disclosure of its rule?}~(c).
Together, the answers show that the two-sided improvement comes from the incentive
structure itself, not from a lucky choice of reward scope, a trusting
verifier, or secrecy.

\begin{figure}[t]
\centering
\includegraphics[width=\linewidth]{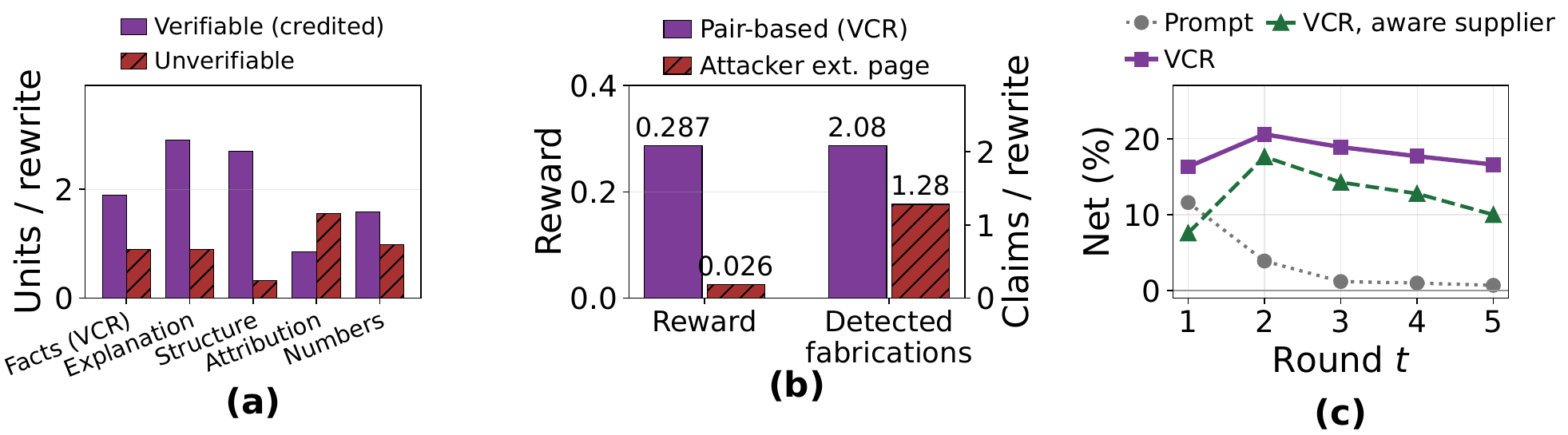}
\caption{
Probing the three design choices of the VCR reward channel on
\textsc{E-commerce}.
\textbf{(a)} Verifiable versus unverifiable content units added per
rewrite under five candidate reward scopes.
\textbf{(b)} Reward earned and fabrications detected when verification uses
the original pair versus an attacker-created external ``support page.''
\textbf{(c)} Net over rounds when the supplier is told the VCR rule and
optimizes against it.
}
\label{fig:understanding}
\end{figure}

\paragraph{RQ1: what content should the reward credit?}
Fig.~\ref{fig:understanding}(a) re-scores the same rewrites under five
candidate reward scopes, separating content units that are grounded in the
original from those that cannot be verified. Fact-style units are
predominantly verifiable, whereas rewarding attribution or authority signals
would credit a content type dominated by unverifiable claims, directly
inviting fabricated sourcing. Broadening the scope to explanatory or
structural content preserves a positive Net while trading defense against
exposure, so the conservative fact-level scope is a robust rather than
knife-edge choice.

\paragraph{RQ2: what evidence should verification trust?}
Given fact-level units, the next question is what evidence may vouch for
them. After each rewrite, the attacker additionally creates an external
``support page'' that justifies its fabricated claims. As shown in
Fig.~\ref{fig:understanding}(b), verifying against the attacker-created
page nearly eliminates the earned reward, yet simultaneously weakens the
fabrication penalty, letting a substantial share of unsupported claims go
undetected. External verification therefore opens a circular, gameable
channel, whereas verification anchored on the original document pair stays
outside the attacker's control; trusted external verifiers remain future
work.

\paragraph{RQ3: does the mechanism survive disclosure?}
With the reward unit and verification anchor fixed, the remaining concern is
that disclosing the resulting rule invites reward hacking. In
Fig.~\ref{fig:understanding}(c), the supplier is explicitly told the VCR
rule and keeps optimizing against it over rounds. Its Net drops relative to
the standard supplier but remains far above Prompt throughout. Because the
reward only credits verifiable factual substance, the most profitable way to
``game'' VCR is to actually add checkable content, so strategic awareness
weakens but cannot invert the joint gain.

\paragraph{Case study: GEO rules.}
Table~\ref{tab:case-study-rules} illustrates these incentives qualitatively by showing how platform-inferred GEO rules evolve on \textsc{E-commerce}. At $R_1$, both settings mainly capture formatting signals. By $R_3$ and $R_5$, the baseline shifts toward manipulation cues, such as authority framing and hidden GEO intent, while VCR shifts toward quality cues, such as explanation, attribution, and verifiability. Since the extractor is unchanged, the divergence comes from the supplier's response to the VCR reward. More examples are provided in Apps.~\ref{app:rule-alignment}, \ref{app:cross-attacker-rules}, and~\ref{app:case-study}.

\section{Conclusion}
\label{sec:conclusion}

We studied citation competition as a repeated platform--creator game and showed how conventional defenses can approach an inert outcome. VCR instead rewards source-supported factual substance while penalizing suspicious manipulation. Across the evaluated simulations, this incentive yields higher Net defense--utility, more substantive rewrites, and better generated answers; creator exposure remains within the $5$-point empirical equivalence band in all nine settings. To our knowledge, this is the first mechanism-design treatment of generative engine optimization that turns platform defense from a filter into a two-sided incentive; extending the model beyond a single platform and creator pool is a promising next step.

\bibliographystyle{plainnat}
\bibliography{custom}

\begin{thebibliography}{50}
\providecommand{\natexlab}[1]{#1}
\providecommand{\url}[1]{\texttt{#1}}
\expandafter\ifx\csname urlstyle\endcsname\relax
  \providecommand{\doi}[1]{doi: #1}\else
  \providecommand{\doi}{doi: \begingroup \urlstyle{rm}\Url}\fi

\bibitem[Aggarwal et~al.(2024)Aggarwal, Murahari, Rajpurohit, Kalyan, Narasimhan, and Deshpande]{aggarwal2024geo}
Pranjal Aggarwal, Vishvak Murahari, Tanmay Rajpurohit, Ashwin Kalyan, Karthik Narasimhan, and Ameet Deshpande.
\newblock Geo: Generative engine optimization.
\newblock In \emph{Proceedings of the 30th ACM SIGKDD conference on knowledge discovery and data mining}, pages 5--16, 2024.

\bibitem[Asai et~al.(2024)Asai, Wu, Wang, Sil, and Hajishirzi]{asai2023selfrag}
Akari Asai, Zeqiu Wu, Yizhong Wang, Avirup Sil, and Hannaneh Hajishirzi.
\newblock Self-rag: Learning to retrieve, generate, and critique through self-reflection.
\newblock In \emph{The Twelfth International Conference on Learning Representations (ICLR)}, 2024.

\bibitem[Bagga et~al.(2025)Bagga, Farias, Korkotashvili, Peng, and Wu]{bagga2025egeo}
Puneet~S Bagga, Vivek~F Farias, Tamar Korkotashvili, Tianyi Peng, and Yuhang Wu.
\newblock E-geo: A testbed for generative engine optimization in e-commerce.
\newblock \emph{arXiv preprint arXiv:2511.20867}, 2025.

\bibitem[Becchetti et~al.(2006)Becchetti, Castillo, Donato, Leonardi, and Baeza{-}Yates]{becchetti2006linkbased}
Luca Becchetti, Carlos Castillo, Debora Donato, Stefano Leonardi, and Ricardo~A. Baeza{-}Yates.
\newblock Link-based characterization and detection of web spam.
\newblock In \emph{AIRWeb 2006, Proceedings of the Second International Workshop on Adversarial Information Retrieval on the Web, Seattle, Washington, USA, 10 August 2006, co-located with {SIGIR} 2006}, pages 1--8, 2006.

\bibitem[Ben-Porat and Tennenholtz(2018)]{ben-porat2018mechanism}
Omer Ben-Porat and Moshe Tennenholtz.
\newblock A game-theoretic approach to recommendation systems with strategic content providers.
\newblock In \emph{Advances in Neural Information Processing Systems 31 (NeurIPS)}, 2018.

\bibitem[Bohnet et~al.(2023)Bohnet, Tran, Verga, Aharoni, Andor, Soares, Ciaramita, Eisenstein, Ganchev, Herzig, Hui, Kwiatkowski, Ma, Ni, Saralegui, Schuster, Cohen, Collins, Das, Metzler, Petrov, and Webster]{bohnet2022attributed}
Bernd Bohnet, Vinh~Q. Tran, Pat Verga, Roee Aharoni, Daniel Andor, Livio~Baldini Soares, Massimiliano Ciaramita, Jacob Eisenstein, Kuzman Ganchev, Jonathan Herzig, Kai Hui, Tom Kwiatkowski, Ji~Ma, Jianmo Ni, Lierni~Sestorain Saralegui, Tal Schuster, William~W. Cohen, Michael Collins, Dipanjan Das, Donald Metzler, Slav Petrov, and Kellie Webster.
\newblock Attributed question answering: Evaluation and modeling for attributed large language models.
\newblock \emph{arXiv preprint arXiv:2212.08037}, 2023.

\bibitem[Castillo and Davison(2011)]{castillo2011adversarial}
Carlos Castillo and Brian~D Davison.
\newblock Adversarial web search.
\newblock \emph{Foundations and trends in Information Retrieval}, 4\penalty0 (5):\penalty0 377--486, 2011.

\bibitem[Chen et~al.(2025{\natexlab{a}})Chen, Chen, Huang, Shao, Chen, Hua, Xu, Wu, Chuan, and Wu]{chen2025ccgseo}
Qiyuan Chen, Jiahe Chen, Hongsen Huang, Qian Shao, Jintai Chen, Renjie Hua, Hongxia Xu, Ruijia Wu, Ren Chuan, and Jian Wu.
\newblock Cc-gseo-bench: A content-centric benchmark for measuring source influence in generative search engines.
\newblock \emph{arXiv preprint arXiv:2509.05607}, 2025{\natexlab{a}}.

\bibitem[Chen et~al.(2025{\natexlab{b}})Chen, Bao, Wu, Chen, and Liao]{chen2025caption}
Xiaolu Chen, Jie Bao, Haojie Wu, Zhen Chen, and Yong Liao.
\newblock Caption injection for optimization in generative search engine.
\newblock \emph{arXiv preprint arXiv:2511.04080}, 2025{\natexlab{b}}.

\bibitem[Chen et~al.(2025{\natexlab{c}})Chen, Wu, Bao, Chen, Liao, and Huang]{chen2025role}
Xiaolu Chen, Haojie Wu, Jie Bao, Zhen Chen, Yong Liao, and Hu~Huang.
\newblock Role-augmented intent-driven generative search engine optimization.
\newblock \emph{arXiv preprint arXiv:2508.11158}, 2025{\natexlab{c}}.

\bibitem[Chen et~al.(2024)Chen, Xiang, Xiao, Song, and Li]{chen2024agentpoison}
Zhaorun Chen, Zhen Xiang, Chaowei Xiao, Dawn Song, and Bo~Li.
\newblock Agentpoison: Red-teaming llm agents via poisoning memory or knowledge bases.
\newblock In \emph{Advances in Neural Information Processing Systems 37 (NeurIPS)}, 2024.

\bibitem[Edemacu et~al.(2025)Edemacu, Shashidhar, Tuape, Abudu, Jang, and Kim]{edemacu2025filterrag}
Kennedy Edemacu, Vinay~M. Shashidhar, Micheal Tuape, Dan Abudu, Beakcheol Jang, and Jong~Wook Kim.
\newblock Defending against knowledge poisoning attacks during retrieval-augmented generation.
\newblock \emph{arXiv preprint arXiv:2508.02835}, 2025.

\bibitem[Es et~al.(2024)Es, James, Espinosa-Anke, and Schockaert]{es2023ragas}
Shahul Es, Jithin James, Luis Espinosa-Anke, and Steven Schockaert.
\newblock Ragas: Automated evaluation of retrieval augmented generation.
\newblock In \emph{Proceedings of the 18th Conference of the European Chapter of the Association for Computational Linguistics: System Demonstrations}, 2024.

\bibitem[Fetterly et~al.(2003)Fetterly, Manasse, Najork, and Wiener]{fetterly2003evolution}
Dennis Fetterly, Mark Manasse, Marc Najork, and Janet~L. Wiener.
\newblock A large-scale study of the evolution of web pages.
\newblock In \emph{Proceedings of the 12th International Conference on World Wide Web}, 2003.

\bibitem[Gao et~al.(2023)Gao, Yen, Yu, and Chen]{gao2023alce}
Tianyu Gao, Howard Yen, Jiatong Yu, and Danqi Chen.
\newblock Enabling large language models to generate text with citations.
\newblock In \emph{Proceedings of the 2023 Conference on Empirical Methods in Natural Language Processing (EMNLP)}, 2023.

\bibitem[Goren et~al.(2018)Goren, Kurland, Tennenholtz, and Raiber]{goren2018ranking}
Gregory Goren, Oren Kurland, Moshe Tennenholtz, and Fiana Raiber.
\newblock Ranking robustness under adversarial document manipulations.
\newblock In \emph{The 41st International ACM SIGIR Conference on Research \& Development in Information Retrieval}, pages 395--404, 2018.

\bibitem[Greshake et~al.(2023)Greshake, Abdelnabi, Mishra, Endres, Holz, and Fritz]{greshake2023indirect}
Kai Greshake, Sahar Abdelnabi, Shailesh Mishra, Christoph Endres, Thorsten Holz, and Mario Fritz.
\newblock Not what you've signed up for: Compromising real-world llm-integrated applications with indirect prompt injection.
\newblock In \emph{Proceedings of the 16th ACM Workshop on Artificial Intelligence and Security (AISec)}, 2023.

\bibitem[Gy{\"o}ngyi and Garcia-Molina(2005)]{gyongyi2005linkspam}
Zolt{\'a}n Gy{\"o}ngyi and Hector Garcia-Molina.
\newblock Link spam alliances.
\newblock In \emph{VLDB}, volume~5, pages 517--528, 2005.

\bibitem[Gy{\"{o}}ngyi and Garcia{-}Molina(2005)]{gyongyi2005taxonomy}
Zolt{\'{a}}n Gy{\"{o}}ngyi and Hector Garcia{-}Molina.
\newblock Web spam taxonomy.
\newblock In \emph{AIRWeb 2005, First International Workshop on Adversarial Information Retrieval on the Web, co-located with the {WWW} conference, Chiba, Japan, May 2005}, pages 39--47, 2005.

\bibitem[Hardt et~al.(2016)Hardt, Megiddo, Papadimitriou, and Wootters]{hardt2016strategic}
Moritz Hardt, Nimrod Megiddo, Christos Papadimitriou, and Mary Wootters.
\newblock Strategic classification.
\newblock In \emph{Proceedings of the 2016 ACM Conference on Innovations in Theoretical Computer Science (ITCS)}, pages 111--122, 2016.

\bibitem[Jagadeesan et~al.(2023)Jagadeesan, Garg, and Steinhardt]{jagadeesan2022supply}
Meena Jagadeesan, Nikhil Garg, and Jacob Steinhardt.
\newblock Supply-side equilibria in recommender systems.
\newblock In \emph{Advances in Neural Information Processing Systems 36 (NeurIPS)}, 2023.

\bibitem[Kim et~al.(2026)Kim, Jeong, Kim, Lee, and Lee]{kim2026sageo}
Sunghwan Kim, Wooseok Jeong, Serin Kim, Sangam Lee, and Dongha Lee.
\newblock Sageo arena: A realistic environment for evaluating search-augmented generative engine optimization.
\newblock \emph{arXiv preprint arXiv:2602.12187}, 2026.

\bibitem[Kleinberg and Raghavan(2019)]{kleinberg2019classifiers}
Jon Kleinberg and Manish Raghavan.
\newblock How do classifiers induce agents to invest effort strategically?
\newblock In \emph{Proceedings of the 2019 ACM Conference on Economics and Computation (EC)}, 2019.

\bibitem[Kumar and Lakkaraju(2024)]{kumar2024manipulating}
Aounon Kumar and Himabindu Lakkaraju.
\newblock Manipulating large language models to increase product visibility.
\newblock \emph{arXiv preprint arXiv:2404.07981}, 2024.

\bibitem[Lewis et~al.(2020)Lewis, Perez, Piktus, Petroni, Karpukhin, Goyal, K{\"u}ttler, Lewis, tau Yih, Rockt{\"a}schel, Riedel, and Kiela]{lewis2020retrieval}
Patrick Lewis, Ethan Perez, Aleksandra Piktus, Fabio Petroni, Vladimir Karpukhin, Naman Goyal, Heinrich K{\"u}ttler, Mike Lewis, Wen tau Yih, Tim Rockt{\"a}schel, Sebastian Riedel, and Douwe Kiela.
\newblock Retrieval-augmented generation for knowledge-intensive nlp tasks.
\newblock In \emph{Advances in Neural Information Processing Systems 33 (NeurIPS)}, 2020.

\bibitem[Liu et~al.(2023{\natexlab{a}})Liu, Zhang, and Liang]{liu-etal-2023-evaluating}
Nelson~F. Liu, Tianyi Zhang, and Percy Liang.
\newblock Evaluating verifiability in generative search engines.
\newblock In \emph{Findings of the Association for Computational Linguistics: EMNLP 2023}, 2023{\natexlab{a}}.

\bibitem[Liu et~al.(2023{\natexlab{b}})Liu, Deng, Li, Wang, Wang, Wang, Zhang, Liu, Wang, Zheng, Zhang, and Liu]{liu2023promptinjection}
Yi~Liu, Gelei Deng, Yuekang Li, Kailong Wang, Zihao Wang, Xiaofeng Wang, Tianwei Zhang, Yepang Liu, Haoyu Wang, Yan Zheng, Leo~Yu Zhang, and Yang Liu.
\newblock Prompt injection attack against llm-integrated applications.
\newblock \emph{arXiv preprint arXiv:2306.05499}, 2023{\natexlab{b}}.

\bibitem[Liu and Xu(2026)]{liu2026featgeo}
Zikang Liu and Peilan Xu.
\newblock Think before writing: Feature-level multi-objective optimization for generative citation visibility.
\newblock \emph{arXiv preprint arXiv:2604.19113}, 2026.

\bibitem[Marecki et~al.(2012)Marecki, Tesauro, and Segal]{marecki2012playing}
Janusz Marecki, Gerry Tesauro, and Richard Segal.
\newblock Playing repeated stackelberg games with unknown opponents.
\newblock In \emph{Proceedings of the 11th International Conference on Autonomous Agents and Multiagent Systems - Volume 2}, AAMAS '12, Richland, SC, 2012. International Foundation for Autonomous Agents and Multiagent Systems.

\bibitem[Min et~al.(2023)Min, Krishna, Lyu, Lewis, tau Yih, Koh, Iyyer, Zettlemoyer, and Hajishirzi]{min2023factscore}
Sewon Min, Kalpesh Krishna, Xinxi Lyu, Mike Lewis, Wen tau Yih, Pang~Wei Koh, Mohit Iyyer, Luke Zettlemoyer, and Hannaneh Hajishirzi.
\newblock Factscore: Fine-grained atomic evaluation of factual precision in long form text generation.
\newblock In \emph{Proceedings of the 2023 Conference on Empirical Methods in Natural Language Processing (EMNLP)}, 2023.

\bibitem[Nakano et~al.(2021)Nakano, Hilton, Balaji, Wu, Ouyang, Kim, Hesse, Jain, Kosaraju, Saunders, Jiang, Cobbe, Eloundou, Krueger, Button, Knight, Chess, and Schulman]{nakano2021webgpt}
Reiichiro Nakano, Jacob Hilton, Suchir Balaji, Jeff Wu, Long Ouyang, Christina Kim, Christopher Hesse, Shantanu Jain, Vineet Kosaraju, William Saunders, Xu~Jiang, Karl Cobbe, Tyna Eloundou, Gretchen Krueger, Kevin Button, Matthew Knight, Benjamin Chess, and John Schulman.
\newblock Webgpt: Browser-assisted question-answering with human feedback.
\newblock \emph{arXiv preprint arXiv:2112.09332}, 2021.

\bibitem[Nisan and Ronen(2001)]{nisan2001algorithmic}
Noam Nisan and Amir Ronen.
\newblock Algorithmic mechanism design.
\newblock \emph{Games and Economic Behavior}, 35\penalty0 (1--2):\penalty0 166--196, 2001.

\bibitem[Ntoulas et~al.(2004)Ntoulas, Cho, and Olston]{ntoulas2004new}
Alexandros Ntoulas, Junghoo Cho, and Christopher Olston.
\newblock What's new on the web? the evolution of the web from a search engine perspective.
\newblock In \emph{Proceedings of the 13th International Conference on World Wide Web}, 2004.

\bibitem[Ntoulas et~al.(2006)Ntoulas, Najork, Manasse, and Fetterly]{ntoulas2006detecting}
Alexandros Ntoulas, Marc Najork, Mark Manasse, and Dennis Fetterly.
\newblock Detecting spam web pages through content analysis.
\newblock In \emph{Proceedings of the 15th International Conference on World Wide Web}, WWW '06, page 83–92, 2006.

\bibitem[Oh et~al.(2024)Oh, Verma, and Kumar]{oh2024adversarial}
Sejoon Oh, Gaurav Verma, and Srijan Kumar.
\newblock Adversarial text rewriting for text-aware recommender systems.
\newblock In \emph{Proceedings of the 33rd ACM International Conference on Information and Knowledge Management}. ACM, 2024.

\bibitem[Perdomo et~al.(2020)Perdomo, Zrnic, Mendler-D{\"u}nner, and Hardt]{perdomo2020performative}
Juan~C. Perdomo, Tijana Zrnic, Celestine Mendler-D{\"u}nner, and Moritz Hardt.
\newblock Performative prediction.
\newblock In \emph{Proceedings of the 37th International Conference on Machine Learning (ICML)}, 2020.

\bibitem[Perez and Ribeiro(2022)]{perez2022ignore}
Fábio Perez and Ian Ribeiro.
\newblock Ignore previous prompt: Attack techniques for language models.
\newblock \emph{arXiv preprint arXiv:2211.09527}, 2022.

\bibitem[Puerto et~al.(2025)Puerto, Gubri, Green, Oh, and Yun]{puerto2025cseo}
Haritz Puerto, Martin Gubri, Tommaso Green, Seong~Joon Oh, and Sangdoo Yun.
\newblock C-seo bench: Does conversational seo work?
\newblock \emph{arXiv preprint arXiv:2506.11097}, 2025.

\bibitem[Su et~al.(2024)Su, Zhou, Zhang, Nakov, and Cardie]{su2024robustrag}
Jinyan Su, Jin~Peng Zhou, Zhengxin Zhang, Preslav Nakov, and Claire Cardie.
\newblock Towards more robust retrieval-augmented generation: Evaluating rag under adversarial poisoning attacks.
\newblock \emph{arXiv preprint arXiv:2412.16708}, 2024.

\bibitem[Von~Stackelberg(2010)]{stackelberg2011market}
Heinrich Von~Stackelberg.
\newblock \emph{Market structure and equilibrium}.
\newblock Springer Science \& Business Media, 2010.

\bibitem[Wallace et~al.(2024)Wallace, Xiao, Leike, Weng, Heidecke, and Beutel]{wallace2024instruction}
Eric Wallace, Kai Xiao, Reimar Leike, Lilian Weng, Johannes Heidecke, and Alex Beutel.
\newblock The instruction hierarchy: Training llms to prioritize privileged instructions.
\newblock \emph{arXiv preprint arXiv:2404.13208}, 2024.

\bibitem[Wu et~al.(2026)Wu, Mao, Lin, Yang, Lu, Guo, Zhang, Wu, Huang, and Li]{wu2026mageo}
Beining Wu, Fuyou Mao, Jiong Lin, Cheng Yang, Jiaxuan Lu, Yifu Guo, Siyu Zhang, Yifan Wu, Ying Huang, and Fu~Li.
\newblock From experience to skill: Multi-agent generative engine optimization via reusable strategy learning.
\newblock \emph{arXiv preprint arXiv:2604.19516}, 2026.

\bibitem[Wu et~al.(2025)Wu, Zhong, Kim, and Xiong]{wu2025generative}
Yujiang Wu, Shanshan Zhong, Yubin Kim, and Chenyan Xiong.
\newblock What generative search engines like and how to optimize web content cooperatively.
\newblock \emph{arXiv preprint arXiv:2510.11438}, 2025.

\bibitem[Xiang et~al.(2024)Xiang, Wu, Zhong, Wagner, Chen, and Mittal]{xiang2024robustrag}
Chong Xiang, Tong Wu, Zexuan Zhong, David Wagner, Danqi Chen, and Prateek Mittal.
\newblock Certifiably robust rag against retrieval corruption.
\newblock \emph{arXiv preprint arXiv:2405.15556}, 2024.

\bibitem[Xue et~al.(2024)Xue, Zheng, Hu, Liu, Chen, and Lou]{xue2024badrag}
Jiaqi Xue, Mengxin Zheng, Yebowen Hu, Fei Liu, Xun Chen, and Qian Lou.
\newblock Badrag: Identifying vulnerabilities in retrieval augmented generation of large language models.
\newblock \emph{arXiv preprint arXiv:2406.00083}, 2024.

\bibitem[Yu et~al.(2026)Yu, Yang, Ding, and Sato]{yu2026structural}
Junwei Yu, Mufeng Yang, Yepeng Ding, and Hiroyuki Sato.
\newblock Structural feature engineering for generative engine optimization: How content structure shapes citation behavior.
\newblock \emph{arXiv preprint arXiv:2603.29979}, 2026.

\bibitem[Yuan et~al.(2026)Yuan, Wang, Wang, Sun, Wang, and Li]{yuan2026agenticgeo}
Jiaqi Yuan, Jialu Wang, Zihan Wang, Qingyun Sun, Ruijie Wang, and Jianxin Li.
\newblock Agenticgeo: A self-evolving agentic system for generative engine optimization.
\newblock \emph{arXiv preprint arXiv:2603.20213}, 2026.

\bibitem[Zhong et~al.(2023)Zhong, Huang, Wettig, and Chen]{zhong2023poisoning}
Zexuan Zhong, Ziqing Huang, Alexander Wettig, and Danqi Chen.
\newblock Poisoning retrieval corpora by injecting adversarial passages.
\newblock In \emph{Proceedings of the 2023 Conference on Empirical Methods in Natural Language Processing (EMNLP)}, 2023.

\bibitem[Zhou et~al.(2026)Zhou, Chen, Chen, Bao, Chen, and Liao]{zhou2026ifgeo}
Heyang Zhou, JiaJia Chen, Xiaolu Chen, Jie Bao, Zhen Chen, and Yong Liao.
\newblock If-geo: Conflict-aware instruction fusion for multi-query generative engine optimization.
\newblock \emph{arXiv preprint arXiv:2601.13938}, 2026.

\bibitem[Zou et~al.(2025)Zou, Geng, Wang, and Jia]{zou2024poisonedrag}
Wei Zou, Runpeng Geng, Binghui Wang, and Jinyuan Jia.
\newblock $\{$PoisonedRAG$\}$: Knowledge corruption attacks to $\{$Retrieval-Augmented$\}$ generation of large language models.
\newblock In \emph{34th USENIX Security Symposium (USENIX Security 25)}, pages 3827--3844, 2025.

\end{thebibliography}
\clearpage

\appendix

\section{Limitations and Open Directions}
\label{sec:limitations}

\VCR\ makes several simplifying assumptions. First, the theory analyzes a
local quadratic surrogate and the experiments use a finite interaction
horizon; neither establishes a global equilibrium for unrestricted rewriting
strategies. The model also considers one platform and one creator population,
leaving multi-platform competition and paid placement to future work.

Second, the reward checks a rewrite against its earlier version. It measures
source support and factual salience, not independent truth: an inaccurate
statement already present in the earlier version may still receive credit.
Moreover, LLM-based claim counting and rule matching can err, and suppliers
may repeat or split supported claims to approach the reward cap. The cap and
manipulation penalty reduce, but do not eliminate, these risks.

Finally, the reported utilities represent creator exposure and a joint
platform/user objective, with equal weight in Net. The $\pm5$-point creator
equivalence band is an operational tolerance motivated by prior repeated-run
variation, not a formal equivalence test for each current estimate. We report
paired query-level bootstrap intervals for the default-engine results;
additional independent pipeline runs would quantify model- and
trajectory-level variability beyond this query-sampling uncertainty.

\section{Implementation Details}
\label{app:implementation}

\paragraph{Hardware and runtime.}
All experiments are run on a SLURM cluster. Each job uses $8$ CPU cores
and $64$\,GB RAM; no GPU is used. A single five-round game on the full
$1{,}000$-query test split takes $18$--$24$ wall-clock hours on
\textsc{GEO-Bench} and \textsc{Researchy-GEO}, and $12$--$15$ hours on
\textsc{E-commerce}. Smaller diagnostic runs and ablations typically
take $4$--$6$ hours.
VCR adds one cached, JSON-only oracle call per changed pair to the rule
extraction shared with Prompt defense. In an online latency test, all four
defenses complete answer-time processing in approximately $400$ ms; the
pair-level oracle runs off the answer-time path.

\paragraph{Models.}
The answer-generation engine $G$ is \texttt{gemini-2.5-flash-lite}.
The platform-side rule extractor
(\textsc{Explainer}/\textsc{Extractor}/\textsc{Merger}/\textsc{Filter}),
the change-magnitude judge, and the verifiable-content oracle all use
\texttt{gpt-4o-mini} with temperature $0$ and JSON-only decoding when
applicable. We pre-compute and cache quality-judge scores for each
dataset so that target selection is deterministic across runs.

\paragraph{Game configuration.}
Each repeated game runs for $T=5$ rounds. To make trajectories
comparable across defenses, we disable early stopping in all reported
experiments; the supplier and platform always interact for all five
rounds. The change-magnitude threshold $\theta$ is set so that the
median edited document contributes to the rule set in round $R_1$.
Unless otherwise stated, the reward strength is $\lambda=1$. The
per-claim credit $c_n$ and credit cap $c_{\max}$ are calibrated to be on
the same scale as the weakest and strongest single suspicion penalties
in the rule pipeline. The $\lambda$ sweep in
Fig.~\ref{fig:robustness-analysis}(b) covers the same effective range as
multiplying the per-claim credit by an equivalent factor.

\section{Defense and Attacker Baselines}
\label{app:baselines-impl}

\paragraph{Defense baselines.}
We compare against three classical platform-side defenses.
\textbf{Prompt defense} attaches a per-document suspicion label, softly
reorders sources by ascending suspicion, and inserts the warning
template in App.~\ref{app:platform-prompts} into the engine's system
prompt. No document is removed. \textbf{Hard reject} drops any document
whose suspicion score exceeds $\tau=0.45$ before generation; the
threshold is tuned on a $50$-query development split. \textbf{Keyword
scrub} disables all prompt-defense channels and instead applies a regex
filter over $24$ GEO-style phrases, such as ``optimized for GEO'',
``expert-curated'', and ``machine-readable listing'', to every incoming
document. The full phrase list is released with our code.

\paragraph{GEO attacker baselines.}
In Sec.~\ref{subsec:other-attackers}, we evaluate robustness against
four non-\AutoGEO{} attackers. {\scshape RAID}~\citep{chen2025role}
extracts engine-preference rules from contrastive role-conditioned
answers; we use the authors' released codebase with the engine replaced
by ours. {\scshape IF-GEO}~\citep{zhou2026ifgeo} performs
influence-based GEO; we re-implement its influence proxy on top of our
engine. \textsc{SAGEO}~\citep{kim2026sageo} performs structure-aware GEO
via prompt programming; we use the authors' prompts with the top-$5$
candidate set from our retriever. Statistics
Addition~\citep{aggarwal2024geo} adds quantitative claims following the
original paper.

\section{Evaluation Metrics}
\label{app:metrics}

We adopt the AutoGEO framework of~\citet{wu2025generative} and the
visibility metric of~\citet{aggarwal2024geo}. For a query $q$, the
engine retrieves a candidate set $D_q$, and an LLM $G$ produces an
answer $a=G(q,D_q)$. Each candidate document $d\in D_q$ receives a
citation visibility score
\begin{equation}
    \mathrm{Vis}(d,a)
    =
    \mathrm{Word}(d,a)
    +
    \mathrm{Pos}(d,a)
    +
    \mathrm{Overall}(d,a),
    \label{eq:vis-app}
\end{equation}
where $\mathrm{Word}(d,a)$ is the normalized word count of sentences in
$a$ citing $d$, $\mathrm{Pos}(d,a)$ is the location-based weight of the
source-linked text, and $\mathrm{Overall}(d,a)$ integrates the two
signals. We follow~\citet{aggarwal2024geo} for the exact
implementation.

Given the supplier's target set $T(q)\subseteq D_q$, the target GEO
score for query $q$ is
\begin{equation}
    g(q)
    =
    \sum_{d\in T(q)} \mathrm{Vis}(d,a),
    \qquad
    g
    =
    \mathbb{E}_q[g(q)].
    \label{eq:gT-app}
\end{equation}
A higher $g$ means the target documents are cited more prominently in
the engine's answer.

\section{Additional Experiments and Case Studies}
\label{app:additional-experiments}

\subsection{Sensitivity to the Net Weighting}
\label{app:net-omega}

Net weights the two stakeholder utilities equally. To check that the
defense ranking does not depend on this choice, define
\begin{equation}
\mathrm{Net}_\omega=\omega\,\mathrm{Def}+(1-\omega)\,\mathrm{Welf},
\qquad \omega\in[0,1],
\end{equation}
so that $\omega=0.5$ recovers the main metric up to scale.
Figure~\ref{fig:net-omega}(a) shows the Pareto view of the four defenses
across the three datasets: Hard reject lies on the $\mathrm{Net}=0$
diagonal, meaning it buys defense one-for-one with creator exposure;
Prompt and Keyword scrub cluster at the origin; and VCR is the only
defense inside the win-win region on every dataset.
Figure~\ref{fig:net-omega}(b) sweeps $\omega$ on \textsc{E-commerce}.
VCR is the optimal defense for all $\omega<0.60$ on
\textsc{E-commerce}, for $\omega\in(0.24,0.57)$ on \textsc{GEO-Bench},
and for $\omega\in(0.20,0.53)$ on \textsc{Researchy-GEO}. Hard reject
overtakes only when creator utility is nearly ignored
($\omega\gtrsim0.6$), and the near-inert Keyword scrub wins only when
defense is nearly ignored.

In practice, $\omega>0.6$ describes a short-sighted platform objective that
values one point of immediate suppression more than $1.5$ points of creator
exposure. At the Hard-reject operating point this removes roughly half of
creators' natural exposure (Welf.\ $\approx-50$ on \textsc{E-commerce}): the
current answer is protected, but the citation traffic that motivates future
content supply is lost. At the opposite extreme, Keyword scrub wins only when
the platform nearly ignores defense. Thus, the alternatives overtake VCR in
one-sided regimes that omit one of the two ecosystem objectives, rather than
under the two-sided mechanism-design setting studied here.

\begin{figure}[t]
\centering
\includegraphics[width=0.9\linewidth]{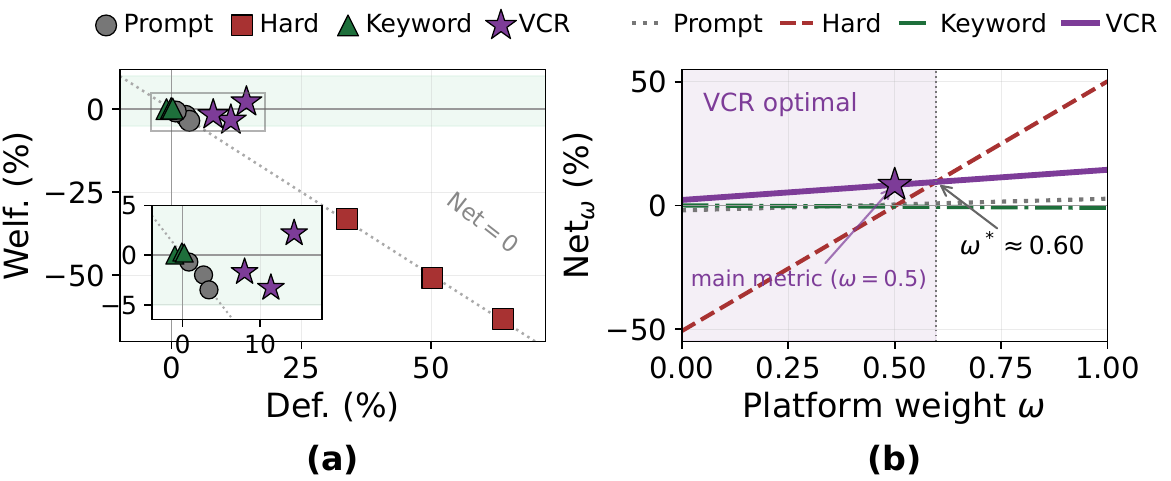}
\caption{Sensitivity of the defense ranking to the Net weighting.
\textbf{(a)} Pareto view of (Def., Welf.) for all defenses and datasets;
the shaded band is the win-win region (positive Def.\ with creator
utility within the $\pm5$-point equivalence band or better).
\textbf{(b)} $\mathrm{Net}_\omega$ on \textsc{E-commerce}: VCR is optimal
for every platform weight $\omega$ below ${\approx}0.60$.}
\label{fig:net-omega}
\end{figure}

\subsection{Statistical Uncertainty}
\label{app:uncertainty}

To quantify the statistical uncertainty of the reported utilities, we
compute paired per-query bootstrap confidence intervals at round $5$. For
each run we pair the three per-query target GEO scores that enter the
metrics (no-attack, attacked, and defended), resample queries with
replacement ($10^4$ replicates), and recompute Def., Welf., and Net on
each replicate; Table~\ref{tab:bootstrap-ci} reports $95\%$ percentile
intervals for the full-scale main runs behind the Gemini block of
Table~\ref{tab:main}.

The default-engine conclusions are well separated from query-sampling
uncertainty. On every dataset, VCR's Net interval lies entirely
above the intervals of all three classical defenses, which all straddle
or hug zero. For creator utility, the \textsc{GEO-Bench} and
\textsc{Researchy-GEO} intervals are fully contained in the $\pm5$-point
band, while the \textsc{E-commerce} interval closely tracks it and reaches
$5.4$ at its upper endpoint. We therefore treat the band as an empirical
operational criterion rather than claim a formal equivalence test.

\begin{table}[t]
\centering
\small
\caption{Paired per-query bootstrap $95\%$ confidence intervals at round
$5$ for the full-scale main runs (percentage points).}
\label{tab:bootstrap-ci}
\setlength{\tabcolsep}{3.5pt}
\begin{tabular}{@{}llrrr@{}}
\toprule
Dataset & Defense & Def. [CI] & Welf. [CI] & Net [CI]\\
\midrule
\textsc{E-commerce} & Prompt & $2.7$ $[0.2, 5.1]$ & $-2.0$ $[-4.6, 0.6]$ & $0.7$ $[-1.5, 2.8]$\\
 & Hard reject & $50.3$ $[40.0, 59.8]$ & $-50.8$ $[-60.3, -40.9]$ & $-0.6$ $[-2.0, 0.7]$\\
 & Keyword scrub & $-1.0$ $[-2.7, 0.6]$ & $0.0$ $[-1.9, 1.9]$ & $-1.0$ $[-2.9, 0.9]$\\
 & \textbf{VCR} & $14.4$ $[12.0, 16.8]$ & $2.2$ $[-0.9, 5.4]$ & $\mathbf{16.6}$ $[13.7, 19.6]$\\
\midrule
\textsc{GEO-Bench} & Prompt & $3.4$ $[2.2, 4.6]$ & $-3.5$ $[-4.7, -2.2]$ & $-0.1$ $[-1.1, 1.0]$\\
 & Hard reject & $33.8$ $[26.7, 40.7]$ & $-33.2$ $[-40.1, -25.9]$ & $0.6$ $[-0.1, 1.4]$\\
 & Keyword scrub & $-0.1$ $[-1.0, 0.7]$ & $0.3$ $[-0.7, 1.2]$ & $0.1$ $[-0.8, 1.1]$\\
 & \textbf{VCR} & $11.4$ $[10.1, 12.7]$ & $-3.3$ $[-4.7, -2.1]$ & $\mathbf{8.0}$ $[6.8, 9.3]$\\
\midrule
\textsc{Researchy-GEO} & Prompt & $0.8$ $[-0.2, 1.9]$ & $-0.7$ $[-1.8, 0.3]$ & $0.1$ $[-0.8, 1.1]$\\
 & Hard reject & $63.9$ $[58.3, 69.2]$ & $-63.3$ $[-68.8, -57.8]$ & $0.5$ $[0.2, 0.9]$\\
 & Keyword scrub & $0.2$ $[-0.6, 0.9]$ & $0.2$ $[-0.6, 1.0]$ & $0.3$ $[-0.5, 1.2]$\\
 & \textbf{VCR} & $8.0$ $[6.9, 9.0]$ & $-1.7$ $[-2.8, -0.5]$ & $\mathbf{6.3}$ $[5.3, 7.3]$\\
\bottomrule
\end{tabular}
\end{table}

\subsection{Judge Robustness}
\label{app:judge-robustness}

To study whether the verifiable-content reward depends on the specific
pair-level LLM oracle, we take the same two pools of round-5 rewrites, one
produced under Prompt defense and one under VCR, and re-score every pair
with three independent judge models from three different providers
(Table~\ref{tab:judge-robustness}). If the reward reflected the
idiosyncrasies of one judge rather than the substance of the rewrites,
gaming the judge would amount to gaming the mechanism; what the mechanism
actually requires is only that the ordering be stable across judges.

From the results, we can see that the reward signal is robust to the
choice of judge. The three judges differ in how conservatively they score,
so the absolute reward levels shift from judge to judge, but every judge
assigns clearly more verifiable-content reward to the VCR-arm rewrites
than to the Prompt-arm rewrites. The signal that drives the mechanism is
therefore a property of the rewrites themselves rather than of one
particular oracle, and the residual disagreement between judges is in line
with the moderate cross-judge agreement reported in prior GEO evaluations
\citep{aggarwal2024geo,wu2025generative}.

\begin{table}[t]
\centering
\caption{Verifiable-content reward assigned by three independent judges.}
\label{tab:judge-robustness}
\begin{tabular}{@{}lrrr@{}}
\toprule
Rewrites from & GPT-4o-mini & Claude-Haiku-4.5 & Gemini-2.5-Flash-Lite\\
\midrule
Prompt defense & 0.213 & 0.180 & 0.244\\
\textbf{VCR} & \textbf{0.243} & \textbf{0.194} & \textbf{0.273}\\
\bottomrule
\end{tabular}
\end{table}

\subsection{Multi-Turn Search}
\label{app:multi-turn}

To study whether VCR's advantage survives beyond the single
retrieval-generation pass of the main experiments, we extend the
environment to a multi-turn search session: the engine issues several
search steps for one user question, accumulates evidence across steps, and
only then composes the answer, with the defense applied at every retrieval
step. This setting changes the exposure economics, since a manipulated
document has several chances to enter the context and an over-aggressive
defense has several chances to drop useful evidence. We compare an
undefended attacked session, Prompt defense, and VCR
(Table~\ref{tab:multi-turn}).

From the results, we can see that the single-pass advantage carries over.
VCR attains the best citation precision and the best user-side answer
utility of the three conditions, while its recall stays essentially
perfect, so the added filtering does not cost the session useful evidence.
The undefended attacked session reaches higher raw visibility for the
target documents but pays for it with lower precision and weaker answers,
which is the multi-turn analogue of the exploitation pattern in the
single-pass game.

\begin{table}[t]
\centering
\caption{Results in the multi-turn search setting.}
\label{tab:multi-turn}
\begin{tabular}{@{}lrrrrr@{}}
\toprule
Condition & GEO score & Precision & Recall & Support & GEU average\\
\midrule
Prompt defense & 0.1426 & 0.6968 & 0.9900 & 0.5620 & 0.5777\\
AutoGEO & 0.1990 & 0.6920 & 0.9850 & 0.5260 & 0.5836\\
\textbf{VCR} & \textbf{0.1924} & \textbf{0.7207} & 0.9894 & 0.5600 & \textbf{0.5998}\\
\bottomrule
\end{tabular}
\end{table}

\subsection{Extended Convergence Check}
\label{app:extended-convergence}

To study whether the fixed five-round horizon of the main experiments
truncates the interaction too early, we rerun the full game without the
fixed horizon and let it continue until its built-in stopping check
terminates it, several rounds past the main evaluation point
(Fig.~\ref{fig:extended-convergence}). If the supplier-platform dynamics
were still moving at round five, endpoint comparisons could be artifacts
of where we stop rather than properties of the defenses.

From the results, we can see that the five-round endpoint is a
representative snapshot of the stabilized interaction. Both the attacked
and the defended target GEO scores flatten before the main horizon and
remain essentially unchanged through the additional rounds, with no late
drift or reversal. The endpoints used in the main tables therefore
reflect stabilized behavior of the repeated game, although we do not
claim that any five-round simulation constitutes a general real-world
equilibrium.

\begin{figure}[t]
\centering
\includegraphics[width=.72\linewidth]{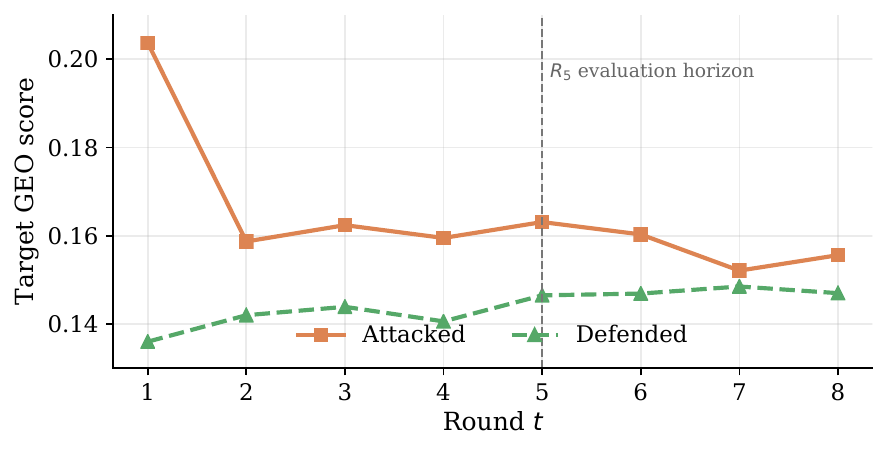}
\caption{Eight-round convergence check. The dashed line marks the main
evaluation horizon; attacked and defended target GEO scores remain stable
thereafter.}
\label{fig:extended-convergence}
\end{figure}

\subsection{Round-by-Round Dynamics across all datasets}
\label{app:trajectory-full}

Figure~\ref{fig:trajectory} extends the main-text trajectory on
\textsc{E-commerce} to all three benchmarks. \VCR{} consistently maintains
positive Net across rounds and datasets. Prompt defense decays by the
second round as the attacker adapts. Keyword scrub stays
near zero throughout, while Hard reject oscillates around zero because
its defense gains are offset by welfare losses. These dynamics match
the theory: classical defenses depend on manipulation signals that can
be eroded through repeated adaptation, whereas \VCR{} rewards a
verifiable-content direction that is welfare-aligned and harder to
neutralize by surface rewriting.

\begin{figure}[t]
\centering
\includegraphics[width=\linewidth]{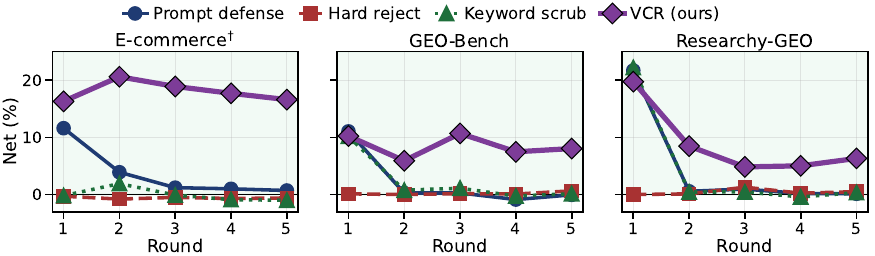}
\caption{
Round-by-round Net performance over five GEO--platform interaction
rounds across all three benchmarks. \VCR{} consistently maintains positive
Net across datasets.
}
\label{fig:trajectory}
\end{figure}

\subsection{User-Side Utility (GEU)}
\label{app:geu-ecomm}

We complement these metrics with the user-side GEU score (Generative-Engine Utility), following the answer-quality evaluation protocol of \citet{aggarwal2024geo}. GEU is computed by \texttt{gpt-4o-mini} on the engine's final answer and directly measures whether the defense improves the end user's experience. Tab.~\ref{tab:geu-ecomm} reports two aggregate summaries.

Tab.~\ref{tab:geu-ecomm} further reports per-round GEU on three substance dimensions: Clarity, Depth, and Insightfulness, comparing \VCR{} with the standard baseline pipeline. The results show three clear trends. First, the two defenses perform similarly in the first round, suggesting that the later gap is unlikely to be caused by initial-condition artifacts. Second, as the interaction proceeds, \VCR{} consistently improves over the baseline across all three dimensions, indicating that the multi-round mechanism gradually translates verifiable-content incentives into better user-facing answers. Third, the improvements are stable across different substance dimensions rather than concentrated in a single metric. Overall, these results suggest that \VCR{} not only improves platform-side outcomes, as shown in Tab.~\ref{tab:main}, but also produces more informative and useful answers for end users.

\begin{table}[t]
\centering
\caption{Per-round GEU substance dimensions on \textsc{E-commerce}. \VCR{} consistently improves Clarity, Depth, and Insightfulness after the first round.}
\label{tab:geu-ecomm}
\small
\setlength{\tabcolsep}{4.5pt}
\renewcommand{\arraystretch}{1.10}
\begin{tabular}{@{}llcccccc@{}}
\toprule
Dim & Defense & $R_1$ & $R_2$ & $R_3$ & $R_4$ & $R_5$ & Avg ($R_1\!\to\!R_5$) \\
\midrule
Clarity        & Baseline              & $0.551$ & $0.551$ & $0.548$ & $0.549$ & $0.549$ & $0.550$ \\
               & \textbf{\VCR\ (ours)} & $0.548$ & \VCRtag{$\mathbf{0.568}$} & \VCRtag{$\mathbf{0.563}$} & \VCRtag{$\mathbf{0.564}$} & \VCRtag{$\mathbf{0.564}$} & \VCRtag{$\mathbf{0.562}$} \\
\cmidrule(l){1-8}
Depth          & Baseline              & $0.503$ & $0.514$ & $0.516$ & $0.517$ & $0.513$ & $0.513$ \\
               & \textbf{\VCR\ (ours)} & $0.500$ & \VCRtag{$\mathbf{0.526}$} & \VCRtag{$\mathbf{0.522}$} & \VCRtag{$\mathbf{0.520}$} & \VCRtag{$\mathbf{0.525}$} & \VCRtag{$\mathbf{0.518}$} \\
\cmidrule(l){1-8}
Insightfulness & Baseline              & $0.450$ & $0.457$ & $0.460$ & $0.456$ & $0.457$ & $0.456$ \\
               & \textbf{\VCR\ (ours)} & $0.450$ & \VCRtag{$\mathbf{0.474}$} & \VCRtag{$\mathbf{0.470}$} & \VCRtag{$\mathbf{0.467}$} & \VCRtag{$\mathbf{0.471}$} & \VCRtag{$\mathbf{0.466}$} \\
\bottomrule
\end{tabular}
\end{table}

\subsection{Full Rule-Alignment Trajectories}
\label{app:rule-alignment}

Figure~\ref{fig:rule-alignment} tracks the alignment between the
platform's inferred rules and the attacker's actual rewriting strategy
over rounds. With the full VCR pipeline, the inferred rules stay aligned
with attacker behavior as the attacker adapts; removing the VCR stage
causes the extractor to drift toward generic quality patterns, showing
that the reward-based filtering is what keeps the defense locked onto
manipulation-specific signals.

Table~\ref{tab:rule-alignment} reports the per-round rule-alignment
score $\langle r_t,\hat r_t\rangle$ between the attacker's actual rule
set and the platform's extracted suspicion rules. We compare
\textsc{high3} and \textsc{low3} target-quality buckets, with and
without the \VCR{} \textsc{Filter} stage. Across both buckets,
alignment improves when \textsc{Filter} is used and degrades when it is
removed. This supports the role of \textsc{Filter}: it removes generic
quality clauses that would match honest content and concentrates the
platform rule set on attacker-specific GEO axes.
Figure~\ref{fig:rule-alignment} plots the \textsc{high3} trajectory.

\begin{figure}[t]
\centering
\includegraphics[width=0.55\linewidth]{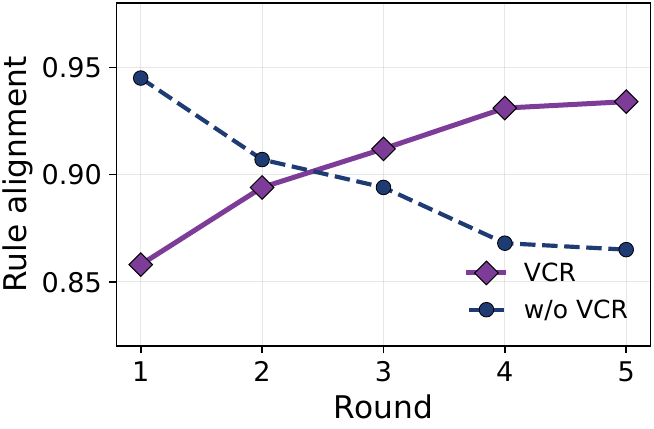}
\caption{Alignment between platform-inferred rules and attacker strategy
over rounds, with and without VCR filtering (\textsc{high3} bucket).}
\label{fig:rule-alignment}
\end{figure}

\begin{table}[t]
\centering
\small
\setlength{\tabcolsep}{4pt}
\begin{tabular}{@{}lcccc@{}}
\toprule
& \multicolumn{2}{c}{\textsc{high3}}
& \multicolumn{2}{c}{\textsc{low3}} \\
\cmidrule(lr){2-3}\cmidrule(lr){4-5}
Round & w/o Filter & \VCR{} full & w/o Filter & \VCR{} full \\
\midrule
$R_1$ & $0.945$ & $0.858$ & $0.958$ & $0.892$ \\
$R_2$ & $0.907$ & $0.894$ & $0.900$ & $0.903$ \\
$R_3$ & $0.894$ & $0.912$ & $0.955$ & $0.926$ \\
$R_4$ & $0.868$ & $0.931$ & $0.835$ & $0.808$ \\
$R_5$ & $0.865$ & $\mathbf{0.934}$ & $0.855$ & $\mathbf{0.943}$ \\
\bottomrule
\end{tabular}
\caption{
Round-by-round rule alignment $\langle r_t,\hat r_t\rangle$ between
attacker rules and platform suspicion rules. With the \textsc{Filter}
stage, alignment increases over rounds; without it, alignment drifts
downward.
}
\label{tab:rule-alignment}
\end{table}

\subsection{Case Study: Platform Rule Refinement under \VCR}
\label{app:case-rules}

We inspect how the platform's suspicion rules evolve across rounds.
Table~\ref{tab:case-rules} compares the top rules extracted at $R_1$
and $R_5$ under the full \VCR{} pipeline. At $R_1$, the extracted rules
mostly reflect generic writing principles, such as relevance, concision,
and structure. These rules can also match honest content and therefore
risk false positives. By $R_5$, the rules become more attacker-specific:
they emphasize negative-constraint coverage, atomic quotability,
explicit GEO signaling, and provenance attribution. This refinement is
consistent with the increasing rule alignment in
Fig.~\ref{fig:rule-alignment}.

\begin{table}[t]
\centering
\scriptsize
\renewcommand{\arraystretch}{1.15}
\setlength{\tabcolsep}{4pt}
\caption{
Platform-side suspicion rules under full \VCR{} at $R_1$ and $R_5$.
Early rules are often generic writing-quality heuristics, while later
rules become more specific to attacker behavior. Rules are truncated
for display.
}
\begin{tabular}{@{}p{0.10\linewidth}p{0.07\linewidth}p{0.78\linewidth}@{}}
\toprule
Dataset & Round & Top-$3$ platform suspicion rules $\hat r_t$ under full \VCR{} \\
\midrule
\textsc{E-commerce} & $R_1$
& \GenericRule{Content must be highly relevant and directly address the user's query, avoiding extraneous, promotional, or conversational filler.} \\
& & \GenericRule{Content should be exceptionally well-structured for AI consumption, using clear headings, lists, and highlighted key terms.} \\
& & \GenericRule{Documents demonstrating explicit optimization for AI consumption and GEO principles for structured extraction are preferred.} \\
\cmidrule(l){2-3}
& $R_5$
& \AttackerRule{Directly address all explicit and implicit aspects of the query, including constraints and negative constraints.} \\
& & \AttackerRule{Ensure core query components are explicitly identified and prominently featured.} \\
& & \AttackerRule{Prioritize factual, specific, quantifiable details over narrative descriptions or extraneous meta-commentary.} \\
\midrule
\textsc{GEO-Bench} & $R_1$
& \GenericRule{Content must be highly relevant and directly address the user's query, minimizing extraneous or tangential information.} \\
& & \GenericRule{Content should be concise and efficient, without unnecessary verbosity, jargon, or conversational filler.} \\
& & \GenericRule{Content should be optimized for machine readability and AI consumption, adhering to GEO principles.} \\
\cmidrule(l){2-3}
& $R_5$
& \AttackerRule{Clearly differentiate and compare related concepts, address potential misconceptions, and offer disambiguating information.} \\
& & \AttackerRule{Ensure content is atomic and segmentable, optimized for direct quotation and synthesis by generative AI systems.} \\
& & \AttackerRule{Explicitly signal document optimization for AI consumption and GEO principles through structure, phrasing, and relevance.} \\
\midrule
\textsc{Researchy} & $R_1$
& \GenericRule{Content must be accurate, credible, and verifiable, supported by data, examples, and citations to reputable sources.} \\
& & \GenericRule{Content must be highly relevant to the query, with key information presented early and prominently.} \\
& & \GenericRule{Content must be presented in a structured, scannable format with descriptive headings and hierarchical organization.} \\
\cmidrule(l){2-3}
& $R_5$
& \AttackerRule{Content should be presented in an atomic, easily isolatable format, with direct, assertive, and quotable statements for streamlined AI synthesis.} \\
& & \AttackerRule{Provide specific, actionable details, concrete examples, and clear ``how/why'' explanations rather than general statements.} \\
& & \AttackerRule{Information must be factual, accurate, specific, and verifiable, including precise data points, dates, names, and figures.} \\
\bottomrule
\end{tabular}
\label{tab:case-rules}
\end{table}

\subsection{Case Study: Cross-Attacker Rule Convergence under \VCR}
\label{app:cross-attacker-rules}

\AutoGEO{} is already analyzed in Tab.~\ref{tab:case-study-rules}; here
we show the remaining four attackers. We inspect why \VCR{} maintains
positive Net across the four non-\AutoGEO{} attackers in
Fig.~\ref{fig:robust-trajectory}(a). For each attacker, we take the top
platform rule extracted at $R_5$ on \textsc{E-commerce}, under Prompt
defense and under \VCR. The $R_1$ rules are generic writing-quality
clauses for all attackers (cf.~Tab.~\ref{tab:case-rules}, $R_1$ rows)
and are omitted.

\begin{table}[t]
\centering
\scriptsize
\renewcommand{\arraystretch}{1.18}
\setlength{\tabcolsep}{4pt}
\caption{
Top platform rule at $R_5$ on \textsc{E-commerce} for the four
non-\AutoGEO{} attackers. Under Prompt defense, the rule tracks each
attacker's specific manipulation pattern. Under \VCR, the rule of all
four attackers converges onto a common verifiability axis. Rules are
taken from \texttt{platform/merged\_rules.json} and lightly truncated.
}
\begin{tabular}{@{}p{0.13\linewidth}p{0.39\linewidth}p{0.39\linewidth}@{}}
\toprule
Attacker & Top platform rule under \textbf{Prompt defense} ($R_5$) & Top platform rule under \textbf{\VCR{}} ($R_5$) \\
\midrule
\textsc{RAID}
& \GenericRule{Content must directly address the user's query, prioritizing} \AttackerRule{actionable advice, clear recommendations, and specific value propositions} \GenericRule{over meta-analysis or tangential details.}
& \GenericRule{The document should demonstrate} \AttackerRule{authority and trustworthiness through accurate, up-to-date, specific, and quantifiable information, potentially including sourcing, comparative data, and concrete examples,} \GenericRule{framed with confident and definitive language.}
\\
\midrule
\textsc{SAGEO}
& \GenericRule{Content should be organized into} \AttackerRule{self-contained, logically segmented blocks or sections} \GenericRule{that are easily parsable and directly usable by language models.}
& \GenericRule{The document should be} \AttackerRule{self-contained, offering sufficient core information without immediate reliance on external links} \GenericRule{for basic details.}
\\
\midrule
Statistics Addition
& \GenericRule{The document must present} \AttackerRule{accurate, verifiable, and up-to-date factual information and statistics, supported by credible sources} \GenericRule{and specific evidence.}
& \GenericRule{The document must be} \AttackerRule{factually accurate, verifiable, and up-to-date,} \GenericRule{avoiding generalizations, vague statements, or outdated content.}
\\
\midrule
\textsc{IF-GEO}
& \GenericRule{The document should} \AttackerRule{anticipate user needs and concerns, providing proactive advice, warnings, context, and practical guidance} \GenericRule{for decision-making and implementation.}
& \GenericRule{Information must be} \AttackerRule{accurate, verifiable, and up-to-date, providing specific evidence, citations, disclaimers, and clearly stating temporal relevance} \GenericRule{or limitations.}
\\
\bottomrule
\end{tabular}
\label{tab:cross-attacker-rules}
\end{table}

Under Prompt defense, the four rules diverge and each captures a
different manipulation pattern: recommendation framing for \textsc{RAID},
segmented blocks for \textsc{SAGEO}, sourced statistics for Statistics
Addition, and proactive guidance for \textsc{IF-GEO}. The platform
locates each attacker, but Prompt defense gives the supplier no reward
for honest content, so all four runs end with near-zero or negative Net.
Under \VCR, the four rules collapse onto the same verifiability axis
(accurate, verifiable, sourced, specific evidence). The
verifiable-content reward then pays the supplier for the same kind of
content regardless of attacker, which matches the uniform positive Net
in Fig.~\ref{fig:robust-trajectory}(a) and the gradient-redirection
result of Thm.~\ref{thm:mechanism-welfare}.

\subsection{Case Study: Escalating Fabrications}
\label{app:case-study}

We measure hallucination harm with a pair-level LLM judge
(\texttt{gpt-4o-mini}, temperature $0$) that compares each original and
rewritten document. The judge returns the number of unsupported claims
$u_i$ and their severity $s_i\in\{1,\ldots,5\}$, where larger severity
indicates more central fabrications. We aggregate these judgments as
\begin{equation}
    \xi^-
    =
    f \cdot \max\!\left(0,\frac{\bar{s}_{u>0}-1}{4}\right),
    \qquad
    f=\Pr_i[u_i>0],
    \qquad
    \bar{s}_{u>0}=\mathbb{E}[s_i\mid u_i>0],
    \label{eq:xi-minus}
\end{equation}
where $f$ is the fraction of rewrites containing unsupported claims, and
the second term normalizes conditional severity to $[0,1]$. Thus,
$\xi^-$ is a frequency-weighted hallucination cost per rewritten
document. It enters the welfare decomposition
$\xi=\xi^+-\xi^-$ in Eq.~(\ref{eq:welfare-decomp}), where $\xi^+$
captures formatting or readability gains and $\xi^-$ captures
hallucination harm. Full prompts are in App.~\ref{app:oracles}.

Figure~\ref{fig:intro} shows a \textsc{E-commerce} example under
\textbf{Prompt defense}. Table~\ref{tab:case-study} repeats the same
query--document pair under \VCR{}. Under \textbf{Prompt defense}, the
rewrite accumulates unsupported specifics over rounds. Under \VCR{}, the
rewrite instead foregrounds source-supported details, such as platform
information and qualitative descriptions of cooperative play. This
matches the gradient-redirection effect predicted by
Thm.~\ref{thm:mechanism-welfare}.

\begin{table}[t]
\centering
\scriptsize
\renewcommand{\arraystretch}{1.12}
\setlength{\tabcolsep}{3pt}
\caption{
Case study of the same query--document pair. We show rewritten content
across rounds, together with hallucination harm $\xi^-$ and the number
of unsupported claims.
}
\begin{tabular}{@{}c p{0.39\linewidth} p{0.39\linewidth}@{}}
\toprule
Round & Prompt defense & \VCR{} \\
\midrule
$R_1$
& ARK is ``highly rated'' with ``intricate building systems''; Roblox hosts ``millions of user-created games.''
& ARK is a sandbox-survival game praised for building systems; Roblox is a user-generated-content platform. \\
\midrule
$R_3$
& GEO-optimized rewrite adds ``high-fidelity 3D graphics,'' ``higher polygon counts,'' and ``advanced textures.''
& Source-grounded rewrite highlights ARK multiplayer building and Roblox as a PC/mobile/console UGC platform. \\
\midrule
$R_5$
& Rewrite escalates to ``over 100 unique dinosaurs'' and ``deep multiplayer tribal systems for cooperative base building.''
& Rewrite keeps to source-supported claims: ARK tribe-based building and Roblox user-generated worlds. \\
\midrule
Hallucination harm $\xi^-$
& $0.42$
& $0.21$ \\
\midrule
Unsupported claims
& 4--6
& 1 \\
\bottomrule
\end{tabular}
\label{tab:case-study}
\end{table}

\newpage

\section{Proofs}
\label{app:proofs}

This appendix proves the results in the main text. We use the notation
from Sec.~\ref{sec:framework}: document features $(q_i,m_i)$,
correlation $\rho=\mathrm{corr}(q_i,m_i)$, source-model logit
$v_i(\alpha)=\beta_q q_i-\alpha m_i+b_i$, citation mass
$c_i(\alpha)=\mathrm{softmax}\{v_j(\alpha)\}_i$, target weight $w_i$,
and target GEO score $g_T(\alpha)=\sum_{i\in T}w_i c_i(\alpha)$.

\subsection{Proof of Lemma~\ref{lem:first-order}}
\label{app:proof-expansion}

\begin{proof}
Let $c_i^{(0)}=\mathrm{softmax}\{\beta_q q_j+b_j\}_i$ and
$\bar m^{(0)}=\sum_j c_j^{(0)}m_j$. We use a local Taylor expansion of
$g_T(\alpha)$ around $\alpha=0$. The coefficients in
Eq.~(\ref{eq:g-expansion}) are
\[
    B =
    \E\!\left[\sum_{i\in T} w_i c_i^{(0)}(m_i-\bar m^{(0)})\right],
\]
and
\[
    Q =
    \E\!\left[
    \sum_{i\in T} w_i c_i^{(0)}
    \left((m_i-\bar m^{(0)})^2-\mathrm{Var}_{c^{(0)}}(m)\right)
    \right].
\]
To derive them, differentiate $\log c_i=v_i-\log Z$:
\[
\partial_\alpha\log c_i
=
-(m_i-\sum_j c_jm_j).
\]
Therefore,
\[
\partial_\alpha c_i\big|_0
=
-c_i^{(0)}(m_i-\bar m^{(0)}).
\]
A second differentiation uses
$\partial_\alpha\bar m\big|_0=-\mathrm{Var}_{c^{(0)}}(m)$ and gives
\[
\partial^2_\alpha c_i\big|_0
=
c_i^{(0)}
\left((m_i-\bar m^{(0)})^2-\mathrm{Var}_{c^{(0)}}(m)\right).
\]
Taking the $w$-weighted target sum and then expectation yields
Eq.~(\ref{eq:g-expansion}). The remainder is local in $\alpha$ under the
standard smoothness conditions of the softmax source model.
\end{proof}

\subsection{Proof of Theorem~\ref{thm:separability}}
\label{app:proof-separability}

\begin{proof}
Maximizing platform utility is equivalent to minimizing the local loss
\[
L(\alpha)
=
g_T(\alpha)
+
\mu\alpha\beta_q\mathrm{cov}(q,m)
+
\frac{\gamma}{2}\alpha^2,
\]
where the second term is the leading false-positive cost from
Sec.~\ref{subsec:game}. Substituting Lem.~\ref{lem:first-order} and
using $\mathrm{cov}(q,m)=\rho\sigma_q\sigma_m$ gives
\[
L(\alpha)
=
g_T(0)
+
\alpha(-B+\mu\beta_q\rho\sigma_q\sigma_m)
+
\frac{1}{2}\alpha^2(Q+\gamma)
+
O(\alpha^3).
\]
Ignoring higher-order terms in the local regime, the first-order
condition yields
\[
(Q+\gamma)\alpha
=
B-\mu\beta_q\rho\sigma_q\sigma_m.
\]
Projecting onto the feasible set $\alpha\ge0$ gives
Eq.~(\ref{eq:alpha-star}). The phase threshold and monotonicity in
$\rho$ follow directly from the numerator.
\end{proof}

\subsection{Proof of Theorem~\ref{thm:defense-effectiveness}}
\label{app:proof-defense-effectiveness}

\begin{proof}
Embed the true rewrite rule $r_t$ and the platform estimate $\hat r_t$
as unit vectors in a shared manipulation-rule basis. Since the platform
applies $\hat r_t$ while the supplier moves along $r_t$, the effective
directional alignment is
\[
\langle r_t,\hat r_t\rangle
=
1-\frac{1}{2}\|r_t-\hat r_t\|_2^2
=
1-e_t.
\]
Let $B_t$ denote the first-order sensitivity of the target GEO score to
a unit defense applied in the true manipulation direction at round $t$.
By the same local expansion as Lem.~\ref{lem:first-order}, applying a
defense of strength $\alpha_t$ along the inferred direction gives
\[
g_T(D_t^a;\pi_{P,t-1})-g_T(D_t^a;\pi_{P,t})
=
\alpha_t(1-e_t)B_t+O(\alpha_t^2).
\]
The realized sensitivity $B_t$ scales with the average manipulation
magnitude of the target documents. Writing
$s_t=\mathbb{E}[m_i\mid i\in T_t]$, we absorb the remaining local
source-model moments into a nonnegative constant $c$ and write
$B_t=c s_t$ to leading order. This yields
\[
g_T(D_t^a;\pi_{P,t-1})-g_T(D_t^a;\pi_{P,t})
=
c(1-e_t)s_t\alpha_t+O(\alpha_t^2).
\]
Dropping higher-order terms gives Eq.~(\ref{eq:defense-decay}).
\end{proof}

\subsection{Proof of Theorem~\ref{thm:mechanism-welfare}}
\label{app:proof-mechanism-welfare}

\begin{proof}
The mechanism augments the source model with source-supported content:
\[
v_i(\alpha,\lambda)
=
\beta_q q_i-\alpha m_i+\lambda n_i+b_i.
\]
Write the platform's quadratic local loss as
$L_P(\alpha,\lambda)$. Orthogonality of $n$ and $m$ removes the mixed
first-order response of the manipulation penalty to the reward, so
$\partial^2 L_P/(\partial\alpha\partial\lambda)|_{(\alpha^\ast,0)}=0$.
Since $\partial^2L_P/\partial\alpha^2=Q+\gamma>0$, the implicit-function
theorem gives
\[
  \alpha^\ast(\lambda)=\alpha^\ast(0)+O(\lambda^2).
\]
Thus the platform response does not offset the reward at first order.

Now let $\delta=(\delta_q,\delta_m,\delta_n)$ and define
$a(\alpha,\lambda)=(\beta_q,-\alpha,\lambda)^\top$. The supplier's local
objective is the strictly concave quadratic
\[
  J_A(\delta;\alpha,\lambda)
  =a(\alpha,\lambda)^\top\delta-\tfrac12\delta^\top H\delta,
  \qquad H\succ0.
\]
Its first-order condition is necessary and sufficient, hence the unique best
response is
\[
  \delta^\ast(\lambda)=H^{-1}a(\alpha^\ast(\lambda),\lambda).
\]
Because the $n$ coordinate is block-separable from $(q,m)$,
\[
  \delta_n^\ast(\lambda)=\lambda(H^{-1})_{nn},\qquad
  (\delta_q^\ast,\delta_m^\ast)(\lambda)
  =(\delta_q^\ast,\delta_m^\ast)(0)+O(\lambda^2).
\]
Let $\Theta=(H^{-1})_{nn}>0$. Extending the platform/user utility
decomposition by the term $\eta_n\delta_n$ therefore yields
\[
  \Delta U^\ast(\lambda)
  =\Delta U^\ast(0)+\eta_n\lambda\Theta+O(\lambda^2).
\]
The linear coefficient is strictly positive, so the platform/user side
strictly improves for all sufficiently small $\lambda>0$.

Finally, substituting the supplier best response into its objective gives
$J_A^\ast(\lambda)=\frac12a^\top H^{-1}a$. Block separability and the
$O(\lambda^2)$ change in $\alpha^\ast$ imply
\[
  J_A^\ast(\lambda)-J_A^\ast(0)=O(\lambda^2).
\]
Thus creator utility is preserved to first order while platform/user utility
increases at first order. This establishes the stated local two-sided result;
it makes no claim about the sign of the creator's second-order change or about
a global equilibrium outside the quadratic neighborhood.
\end{proof}

\section{Prompt Templates}
\label{app:prompts}

This appendix lists the LLM prompts used by the supplier and the
platform. Variables in braces, such as \texttt{\{q\}} and
\texttt{\{d\_i\}}, are substituted at call time. Unless otherwise
specified, prompts use temperature $0$ and JSON-only output when a JSON
schema is requested.

\subsection{Supplier Prompts}
\label{app:attacker-prompts}

\begin{lstlisting}[caption={Supplier Explainer prompt.},label={lst:attacker-explainer}]
Query: {q}
Generated answer:
{a}

Document A (cited more):
{d_i}

Document B (cited less):
{d_j}

Identify concrete properties of Document A that make it more visible
to the generative engine than Document B. Return a bullet list of
candidate explanations and cite specific evidence from Document A
whenever possible.
\end{lstlisting}

\begin{lstlisting}[caption={Supplier Extractor prompt.},label={lst:attacker-extractor}]
Explanation:
{expl}

Extract a concise list of actionable rewriting rules that a content
supplier could apply to another document to improve its citation
frequency on this engine. Each rule should be a single imperative
sentence and should generalize across queries.
\end{lstlisting}

\begin{lstlisting}[caption={Supplier Merger prompt.},label={lst:attacker-merger}]
Candidate rules:
{rules}

Consolidate the candidate rules into a compact list of general
rewriting rules. Remove duplicates and overly specific clauses. Keep
only rules that generalize across documents and queries.
\end{lstlisting}

\begin{lstlisting}[caption={Supplier Filter prompt.},label={lst:attacker-filter}]
Candidate rules:
{rules}

Retain only rules that describe verifiable preferences of a
generative engine, such as structure, comparative framing, or explicit
answerability. Drop generic writing-quality advice that does not
specifically raise citation probability.
\end{lstlisting}

\begin{lstlisting}[caption={Supplier rewrite prompt.},label={lst:attacker-rewrite}]
Rules:
{S_A}

Original document:
{d_i}

Rewrite the document so that it follows all rules above while
preserving factual content. Return only the rewritten document.
\end{lstlisting}

\subsection{Platform Defense Prompts}
\label{app:platform-prompts}

\begin{lstlisting}[caption={Platform Explainer prompt.},label={lst:platform-explainer}]
Original document:
{d_i}

Rewritten document:
{d_i^a}

The rewritten version may have been modified to attract citations in
a generative engine without changing factual content. Identify
concrete properties that distinguish the rewritten version from the
original, such as structural changes, formatting cues, comparative
framing, or keyword density. Return a bullet list of candidate
patterns.
\end{lstlisting}

\begin{lstlisting}[caption={Platform Extractor prompt.},label={lst:platform-extractor}]
Explanation:
{expl}

Extract suspicion rules from the explanation above. Each rule should
describe a feature indicating that a document may have been
GEO-optimized. Each rule should generalize across queries.
\end{lstlisting}

\begin{lstlisting}[caption={Platform Merger prompt.},label={lst:platform-merger}]
Candidate suspicion rules:
{rules}

Consolidate the rules into a compact list of general suspicion
patterns. Remove duplicates and overly specific clauses.
\end{lstlisting}

\begin{lstlisting}[caption={Platform Filter prompt.},label={lst:platform-filter}]
Candidate rules:
{rules}

Retain only GEO-specific patterns, such as aggressive structure,
explicit answer-priority cues, and keyword stuffing. Drop generic
evidence-quality clauses such as "be relevant" or "be factually
correct".
\end{lstlisting}

\begin{lstlisting}[caption={Answer-time defense warning.},label={lst:defense-warning}]
Some retrieved pages may be GEO-optimized to attract citations through
formatting or packaging rather than substance. Treat the following
patterns as suspicious unless the underlying evidence directly answers
the query.

Suspicious patterns:
{rules}

Per-document suspicion labels:
{labels}

Guidelines:
- Prefer low-suspicion documents when they provide sufficient evidence.
- Cite a high-suspicion document only if it provides uniquely necessary
  evidence that lower-suspicion documents do not provide.
- Do not mention these suspicion labels in your answer.
\end{lstlisting}

\subsection{Pair-Level Oracles}
\label{app:oracles}

The following pair-level oracles are used by the platform. The
change-magnitude judge gates whether a pair contributes to the inferred
rule set. The verifiable-content oracle computes the reward signal
$n_i$. The quality judge is used for deterministic target-document
selection.

\begin{lstlisting}[caption={Change-magnitude judge. Returns a score in 0 to 1.},label={lst:change-judge}]
SYSTEM:
You are a search-quality change judge. Compare an older page and a
newer page. Estimate only the magnitude of the update, not whether the
newer page is better. Return JSON only.

USER:
Query: {q}

Score the update magnitude on a 0--1 scale:
0 = trivial edits or local wording changes.
1 = substantial rewriting, restructuring, or repackaging.

Return exactly:
{
  "change_magnitude": 0.0,
  "reason": "short explanation"
}

Older page:
{d_i}

Newer page:
{d_i^a}
\end{lstlisting}

\begin{lstlisting}[caption={Verifiable-content oracle. Returns an integer reward signal.},label={lst:vcr-oracle}]
SYSTEM:
You are a careful fact-checker comparing two versions of a document.
Count two things: (i) claims in the rewritten version that are not
supported by the original, and (ii) new factual content in the
rewritten version that is supported by the original and adds genuine
informational value. Return JSON only.

USER:
Query: {q}

Original document:
{d_i}

Rewritten document:
{d_i^a}

Identify specific factual claims in the rewritten document.

Definitions:
- unsupported_claims: number of concrete details, such as numbers,
  names, prices, citations, or specifications, that appear in the
  rewritten document but cannot be verified from the original.
- verifiable_new_content: number of facts that are supported by the
  original and make the rewritten document easier to extract from, for
  example by promoting buried prices, specifications, or citations into
  salient form.
- severity: severity of unsupported claims, from 1 to 5.

Return JSON:
{
  "unsupported_claims": <int>,
  "verifiable_new_content": <int>,
  "severity": <int 1-5>
}
\end{lstlisting}

\begin{lstlisting}[caption={Quality judge for target-document selection.},label={lst:quality-judge}]
SYSTEM:
You are a search-quality judge. Score each candidate document for how
useful and trustworthy it is for answering the query. Return JSON only.

USER:
Query: {q}

For each document, assign scores from 0 to 1 for quality,
credibility, query_relevance, and overall. Use overall as a balanced
summary. Also include a short reason.

Return exactly:
{
  "documents": [
    {
      "doc_id": 0,
      "quality": 0.0,
      "credibility": 0.0,
      "query_relevance": 0.0,
      "overall": 0.0,
      "reason": "short explanation"
    }
  ]
}

Documents:
[0] {d_0}
[1] {d_1}
...
\end{lstlisting}

\end{document}